\documentclass[11pt,a4paper]{article}

\usepackage[utf8]{inputenc}
\usepackage[T1]{fontenc}
\usepackage{amsmath,amssymb,amsthm}
\usepackage{booktabs}
\usepackage{hyperref}
\usepackage{url}
\usepackage{graphicx}
\usepackage{listings}
\usepackage{xcolor}
\usepackage[margin=1in]{geometry}
\usepackage{natbib}
\usepackage{enumitem}
\usepackage{tabularx}
\usepackage{fancyvrb}
\usepackage{tikz}
\usetikzlibrary{shapes.geometric, arrows.meta, positioning, fit, backgrounds}

\hypersetup{
  colorlinks=true,
  linkcolor=blue,
  citecolor=blue,
  urlcolor=blue
}

\newtheorem{definition}{Definition}
\newtheorem{proposition}{Proposition}

\title{ContextNest: Verifiable Context Governance\\for Autonomous AI Agents}

\author{
  Misha Sulpovar\\
  PromptOwl, LLC\\
  \texttt{misha@promptowl.ai}
  \and
  Benn R. Konsynski\\
  Goizueta Business School, Emory University\\
  \texttt{benn.konsynski@emory.edu}\\
  ORCID: \texttt{0009-0008-3884-9460}
  \and
  Qaish Kanchwala\\
  Independent\\
  \texttt{qaish.kanchwala@gmail.com}
  \and
  Gabe Goodhart\\
  IBM Research\\
  \texttt{ghart@us.ibm.com}
}

\date{June 15, 2026}

\begin{document}

\maketitle

\begin{abstract}
Autonomous AI agents increasingly depend on external knowledge stores, yet most retrieval pipelines provide relevance without durable guarantees of provenance, version identity, integrity, traceability, or point-in-time reconstruction. We formalize this problem as \emph{context governance} and present ContextNest, an open specification and reference implementation for governed AI-consumable knowledge vaults. \textbf{ContextNest is not a replacement for Retrieval-Augmented Generation (RAG); it is a governance layer beneath retrieval.} In the natural composition, ContextNest determines which artifacts are approved, current, attributable, and integrity-verified, while RAG and adjacent retrieval systems operate over that governed subset (§\ref{sec:complementary}).

The specification combines typed Markdown documents with structured metadata, a deterministic set-algebraic selector grammar, addressable \texttt{contextnest://} URI references, SHA-256 hash-chained version histories, graph-level checkpoints, source nodes that integrate live data via the Model Context Protocol (MCP), and audit traces of agent context consumption. Together these mechanisms let an organization reconstruct which knowledge versions informed an agent output and whether those versions were eligible for AI use at the time of consumption.

We report first empirical results from two controlled experiments. In a stale-version attack designed to isolate the governance-versus-retrieval failure mode, governed selection strictly Pareto-dominates BM25 sparse retrieval, achieving a higher answer-quality pass rate (97\% vs.\ 93--90\%) at approximately one-third the input-token cost. In a retrieval-determinism experiment over a synthesized $1{,}060$-document corpus, deterministic selectors and BM25 return perfectly stable document sets across repeated identical queries (Jaccard $1.0$), while a dense\,+\,HNSW baseline configured with realistic production parameters is non-deterministic on $80\%$ of queries (mean Jaccard $0.611$, worst-case $0.210$). These results suggest that context governance addresses a class of failure modes that retrieval quality alone is not designed to resolve. We release a reference implementation---a core engine, command-line interface, and MCP server---under open licenses.
\end{abstract}


\section{Introduction}
\label{sec:intro}

Autonomous AI agents are beginning to move beyond text generation into operational roles: answering policy questions, drafting decisions, coordinating workflows, invoking tools, and acting on behalf of organizations. In these settings, the quality of an agent's behavior depends not only on the model, but on the trustworthiness of the knowledge it consumes. Yet the knowledge layer beneath most agent systems remains weakly governed. Documents are embedded, chunked, retrieved, and injected into prompts, but the resulting pipeline often cannot answer basic questions: Which source version informed this output? Who authored and approved it? Was it current at the time of use? Has it been modified since? Could the same context be reconstructed for audit or replay?

\paragraph{A motivating vignette.}
A procurement agent approves a vendor based on a risk-threshold policy retrieved from the organization's knowledge base. Six months later, an auditor flags the approval: the threshold that governed the decision was revised eighteen months ago, but the retrieval pipeline surfaced an archived version that lived alongside the current one in the same vector index. The agent answered confidently; the index did not distinguish the live policy from its predecessor; no audit trail records which version the agent actually consumed. The agent did not fail at \emph{retrieval}---the document it returned was textually relevant. It failed at \emph{governance}: the system had no notion of which version was approved for AI consumption, no integrity check on what was retrieved, and no reconstructible record of the knowledge that informed the action. Variations of this pattern---a policy agent citing an archived travel reimbursement rule, a support agent quoting a deprecated SLA, a code-review agent flagging a violation of a superseded style guide---are the failure mode this paper addresses.

This is the \textbf{context governance gap} (\textbf{CGG}): the difference between giving an AI system access to information and ensuring that the information it consumes is approved, current, attributable, versioned, tamper-evident, and auditable. Retrieval-Augmented Generation~\citep{lewis2020retrieval} has made external knowledge usable by language models, but \emph{retrieval is not governance}. A vector index can help find relevant passages; it does not, as an abstraction, guarantee provenance, version identity, integrity verification, deterministic selection, traceability, or point-in-time reconstruction.

ContextNest addresses this gap by treating context as a first-class governed artifact: a portable specification for structured, versioned, and cryptographically verifiable knowledge vaults designed for AI-agent consumption. \textbf{The specification scopes the governance layer; it is not a replacement for retrieval.} The natural composition is governed selection over the published, integrity-verified subset of the vault, with retrieval systems (RAG and adjacent) operating on that governed substrate; §\ref{sec:complementary} states this complementarity in full. The specification's mechanisms are: typed Markdown documents with structured metadata, deterministic selector queries, stable addressable references, hash-chained version histories, graph-level checkpoints, and audit traces of context injection.

\paragraph{Thesis.}
\textbf{As AI agents become operational actors, context governance---not model capability alone---becomes the binding constraint on trustworthy enterprise AI.}

\paragraph{Intellectual lineage.}
This thesis is the contemporary form of a question the Information Systems discipline has been refining for thirty-five years: as cognitive labor is progressively offloaded from humans to systems, what guarantees must the systems provide for that offload to be trustworthy? The progression runs from the delegation of perception and analysis to a ``community of intelligent agents''~\citep{elofson1991delegation}, to the dynamic apportionment of cognitive load between human and machine \emph{co-cognitors}~\citep{fjeldstad1986apportionment}, to the formal reapportionment of managerial judgment and decision rights~\citep{sviokla1994reapportionment, konsynski2024cognitive}. Each stage of this lineage has identified the same enabling factor: a trustworthy substrate of grounded, attributable, auditable knowledge. ContextNest is the artifact-level mechanism for that substrate at the inference-time knowledge layer, in the era of generative-AI agents.

\paragraph{Contributions.}
This paper makes the following contributions:

\begin{itemize}[nosep]
  \item A \textbf{formal model of context governance}: six properties that distinguish a governed context system from an undifferentiated document store (§\ref{sec:problem}), together with a precise threat model under which the integrity claims hold (§\ref{sec:threat}).
  \item A \textbf{portable specification} for governed AI-consumable knowledge vaults: typed Markdown documents with YAML frontmatter, a hierarchical stewardship layer with explicit separation of duties (§\ref{sec:stewardship}), a deterministic set-algebraic selector grammar (§\ref{sec:selectors}), an addressable URI scheme with version pinning (§\ref{sec:uri}), SHA-256 hash-chained version histories (§\ref{sec:integrity}), graph-level checkpoints (§\ref{sec:checkpoints}), and source nodes for live-data integration (§\ref{sec:sources}).
  \item A \textbf{reference implementation} (§\ref{sec:implementation}): a core engine, a CLI with nineteen commands, and a Model Context Protocol~\citep{anthropic2024mcp} server, openly licensed and reproducibly deployable.
  \item \textbf{First empirical results} (§\ref{sec:first-results}--§\ref{sec:determinism}): a 30-query stale-version attack designed to isolate the governance-vs-retrieval failure mode (governed selection strictly Pareto-dominates BM25 sparse retrieval on this suite, with two distinct BM25 failure modes---stale-version poisoning and similarity-retrieval miss---demonstrated in the same run), and a 50-query retrieval-determinism experiment against a $1{,}060$-document synthesized corpus (selector and BM25 perfectly deterministic on every query; dense+HNSW non-deterministic on $80\%$ of queries, mean Jaccard $0.611$, worst-case $0.210$).
\end{itemize}

\begin{figure}[t]
\centering
\includegraphics[width=0.92\linewidth]{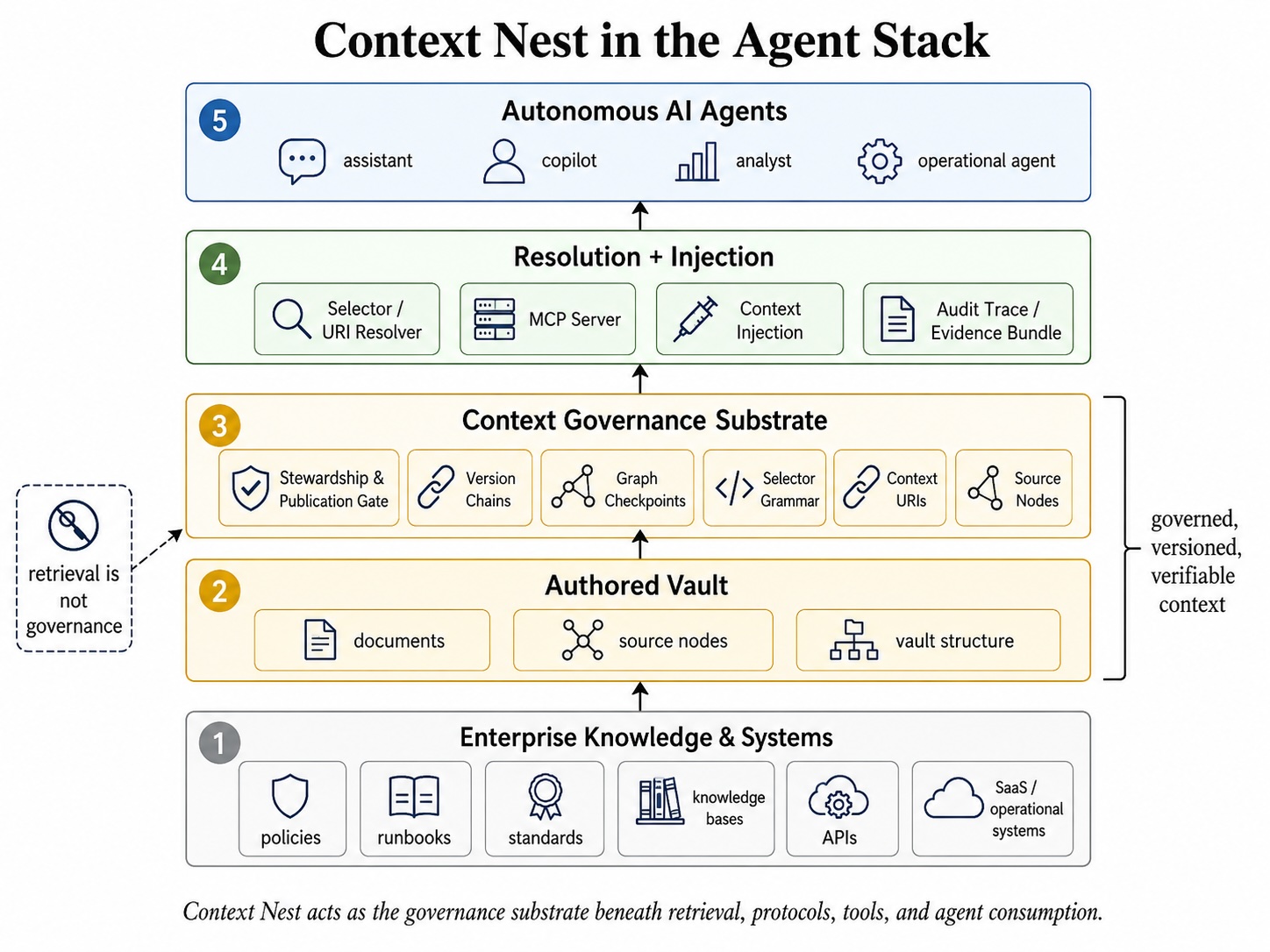}
\caption{ContextNest in the agent stack. The authored vault and the context-governance substrate sit between enterprise knowledge systems below and resolution, injection, and autonomous agents above. ContextNest determines which artifacts are eligible for AI consumption, at which version, under which stewardship, with which integrity guarantees, and with what audit-traceable record; retrieval pipelines operate over the governed subset, and the audit trace flows back up to whichever agent consumed the context---\emph{retrieval is not governance}.}
\label{fig:agent-stack}
\end{figure}

\paragraph{Scope.}
This paper specifies inference-time knowledge governance. \textbf{ContextNest governs the knowledge basis of agent action; it does not govern the full action-authorization stack.} The two boundaries are deliberate. Knowledge governance --- which artifact, at which version, attributed to which steward, with what integrity guarantees --- is the substrate on which any action-authorization layer must rest, because the lawfulness or appropriateness of an agent's action is partly a function of the knowledge that justified it. Action governance proper --- agent identity, owner attestation, capability delegation, runtime policy enforcement on tool calls, and the authorization of outcomes --- is a complementary layer scoped to future work (§\ref{sec:future-work}). Training-data governance, prevention (as distinct from detection) of tampering, real-time multi-user collaboration, and the inter-vault federation protocol are likewise scoped to future work.

\subsection{ContextNest and RAG are Complementary}
\label{sec:complementary}

Throughout this paper we contrast ContextNest with Retrieval-Augmented Generation. The contrast is architectural, not adversarial. We make the relationship explicit here because misreading it weakens both systems.

RAG and ContextNest answer different questions over the same underlying corpus. RAG answers \emph{``which passages are relevant to this query?''}---a retrieval question that benefits from semantic similarity, dense embeddings, and approximate-nearest-neighbor indices. ContextNest answers \emph{``which documents are approved, current, attributable, and integrity-verified, and may an agent consume them right now?''}---a governance question that benefits from typed metadata, stewardship state, hash-chained version history, and deterministic selection.

A production system can---and in many enterprise contexts \emph{should}---use both. The natural composition is: ContextNest governs which documents are eligible for AI consumption (the published, integrity-verified subset of the vault); RAG indexes that governed subset for similarity search; the agent receives both the retrieved passages and the audit trace (§\ref{sec:injection}) that records exactly which governed versions informed its context window. The result is semantic retrieval over a governed substrate, with traceability and reproducibility properties that neither system provides alone.

\paragraph{From personal context to a governed organizational asset.}
The deployment we envision is less a plumbing change than an organizational one. In most enterprises today, the context that conditions an AI system is fragmented across individuals, locked to a particular model or vendor, and discarded after each session---it is personal and ephemeral rather than institutional. ContextNest is the substrate that makes organizational knowledge a governed, portable, model-agnostic asset: one that survives model upgrades, vendor changes, and staff turnover, because provenance, version identity, and eligibility for AI use are properties of the artifact rather than of whoever happened to prompt the model. Concretely, an organization places ContextNest beside its existing document repositories and \emph{before} its retrieval index. The adoption pipeline is: existing corpus $\rightarrow$ ContextNest vault and governance gate (typing, stewardship, publication) $\rightarrow$ RAG index over the published, current subset $\rightarrow$ MCP/agent injection $\rightarrow$ audit trace and evidence bundle $\rightarrow$ point-in-time reconstruction. Retrieval quality is unchanged by this arrangement; what changes is that every indexed artifact is an approved, versioned, attributable node, and every consumption is recorded against the checkpoint in force at the time.

The claim of this paper is therefore not that RAG is wrong, but that retrieval and governance are different functions, and conflating them leaves the governance function unaddressed. The framing is consistent with the Cognitive Apportionment principle of Fjeldstad and Konsynski~\citep{fjeldstad1986apportionment}: in a partnered human-machine cognitive system, the question is not which single mechanism wins but which processor (here: which subsystem) is best equipped to handle each function. Retrieval is one processor; governance is another; the model is the execution layer for both~\citep{konsynski2024cognitive}. Section~\ref{sec:problem} formalizes the governance function as six properties; the remainder of the paper specifies a system that provides them.

\subsection{Architectural Overview}
\label{sec:architecture-overview}

Before specifying the mechanisms in detail, we sketch the architectural shape so that the body sections do not read as a sequence of unmotivated components. Figure~\ref{fig:agent-stack} situates ContextNest within the surrounding agent stack---enterprise knowledge systems below, agent resolution and injection above---and Figure~\ref{fig:governed-flow} traces the governed-context flow, from authoring and steward approval through to agent consumption and audit, that the components below realize. ContextNest is composed of five interlocking parts.

\emph{Typed documents.} The atomic unit of the vault is a Markdown document with a YAML frontmatter header (§\ref{sec:document}). Each document carries a node \emph{type} drawn from a small fixed set---\texttt{document}, \texttt{snippet}, \texttt{glossary}, \texttt{persona}, \texttt{prompt}, \texttt{source}, \texttt{tool}, \texttt{reference} (Table~\ref{tab:nodetypes})---that classifies the document's role in the knowledge graph and enables type-selective queries (e.g.\ ``all glossary entries,'' ``all source nodes''). Types are not a taxonomy of subject matter; they are a taxonomy of \emph{what an agent should do with the document}.

\emph{A knowledge graph.} Cross-document references are expressed as \texttt{contextnest://} URIs in document bodies (§\ref{sec:uri}); tags in frontmatter induce additional set-membership edges; source nodes declare dependency edges to other source nodes. The vault is therefore a directed graph whose nodes are typed documents and whose edges are URI references, tag co-membership, and explicit dependencies. The graph is implicit in the source files and materialized on demand by the engine; it is not maintained as a separate database. The selector grammar (§\ref{sec:selectors}) is the algebra over this graph.

\emph{A stewardship layer.} Above the document and graph layers, ContextNest defines a hierarchical stewardship model (§\ref{sec:stewardship}) that binds principals to scopes (document, folder, tag, vault) at one of three roles (Viewer, Editor, Reviewer). The stewardship layer answers the question \emph{which principal authorized this version's eligibility for AI consumption}, and is the seat of the separation-of-duties rule.

\emph{Hash-chained history and checkpoints.} Every document carries an append-only version history with a SHA-256 chain hash over content, author, timestamp, and ordinal (§\ref{sec:integrity}); the vault as a whole carries an append-only checkpoint log with cross-chain binding to the per-document histories (§\ref{sec:checkpoints}). Together these provide tamper-evidence at both the document level and the graph level, and enable temporal reconstruction of the graph state at any past checkpoint.

\emph{Injection and audit trace.} Agents interact with the vault through a context-injection protocol (§\ref{sec:injection}) that emits a structured audit record for every document or source-node access, identifying the consumed version and the checkpoint at the time of access. The audit trace is the surface through which the governance guarantees of the body sections become visible to the consumer.

\begin{figure}[t]
\centering
\includegraphics[width=0.95\linewidth]{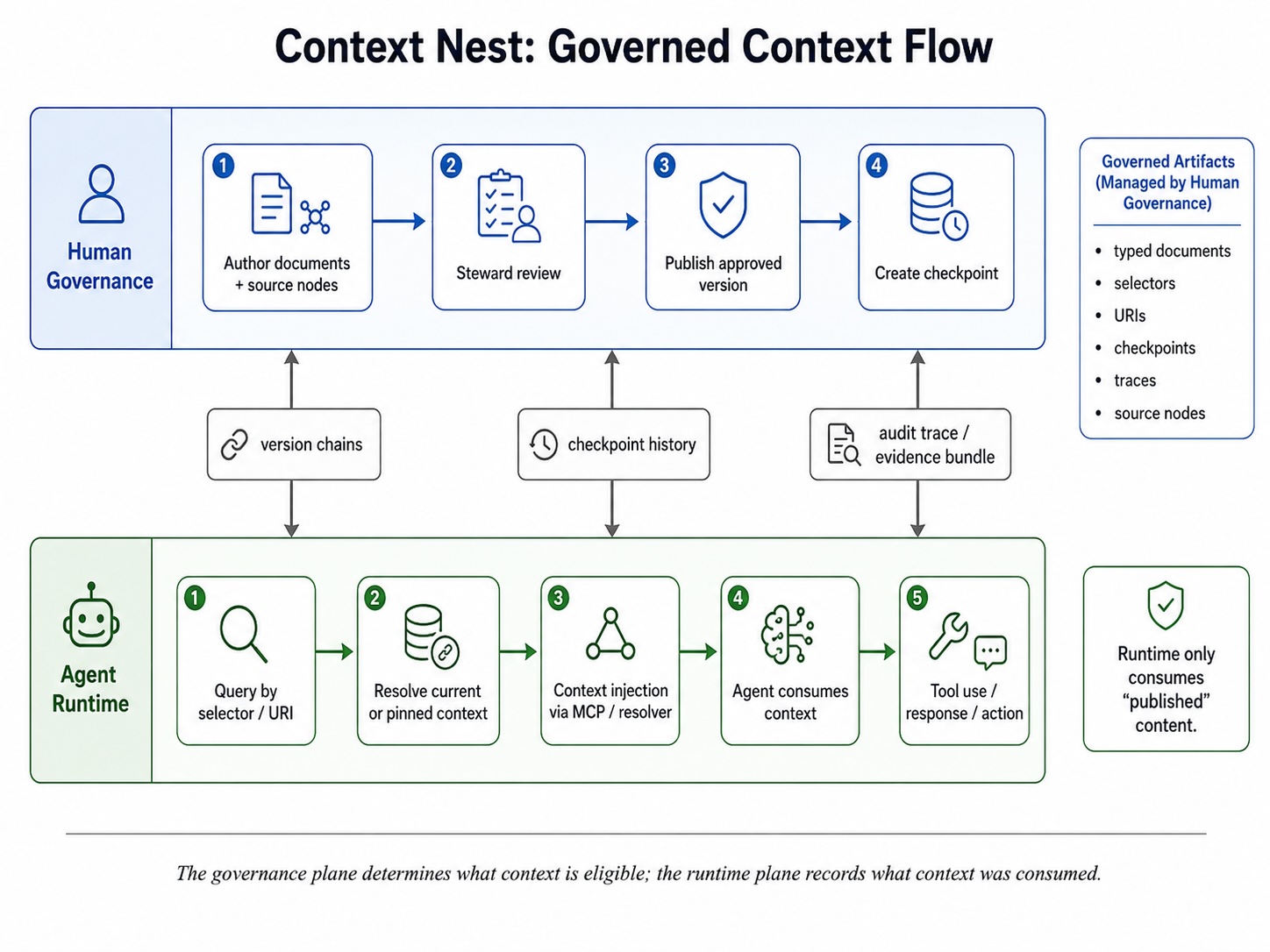}
\caption{Governed context flow. A human-governance plane (author $\rightarrow$ steward review $\rightarrow$ publish approved version $\rightarrow$ checkpoint) produces the governed artifacts---typed documents, selectors, URIs, checkpoints, traces, and source nodes---over which an agent-runtime plane resolves, injects, and consumes context, recording an audit trace at each step. The governance plane determines what context is eligible; the runtime plane records what context was consumed. The runtime consumes only \emph{published} content.}
\label{fig:governed-flow}
\end{figure}

The remainder of the paper specifies each layer in turn. Section~\ref{sec:related} surveys related work. Section~\ref{sec:problem} formalizes context governance as six properties (Definition~\ref{def:governance}) and presents the threat model. Section~\ref{sec:document} covers the document model and the stewardship layer. Section~\ref{sec:selectors} defines the selector grammar; Section~\ref{sec:uri} describes the addressable context (URI) scheme. Sections~\ref{sec:integrity}--\ref{sec:checkpoints} detail integrity verification and checkpoints. Section~\ref{sec:injection} describes injection and tracing, including sub-document selection for large documents. Section~\ref{sec:sources} covers live-data integration via source nodes. Section~\ref{sec:implementation} describes the reference implementation and reports the empirical results. Section~\ref{sec:discussion} discusses limitations and future work.

\section{Related Work}
\label{sec:related}

\subsection{Retrieval-Augmented Generation}

RAG \citep{lewis2020retrieval} augments language model generation with passages retrieved from an external corpus. The standard RAG pipeline embeds documents as vectors, retrieves the top-$k$ most similar passages to a query, and conditions the model's generation on the retrieved context. Subsequent work has improved retrieval quality through better embedding models \citep{izacard2022unsupervised}, hybrid sparse-dense retrieval \citep{chen2024benchmarking}, re-ranking \citep{nogueira2019passage}, and query decomposition strategies \citep{press2023measuring}.

While these advances improve \emph{retrieval relevance}, they do not, as an abstraction, address \emph{retrieval governance}. The RAG pipeline provides no architectural mechanism for tracking document authorship, enforcing version control, detecting post-ingestion tampering, or producing audit trails of which specific document versions informed a given output. Production deployments may layer metadata, snapshots, access controls, or versioned indices atop the abstraction, but those additions are governance infrastructure \emph{external} to the retrieval mechanism itself. The claim of this paper is precise on this point: governance must be supplied by a layer the retrieval abstraction does not provide.

\subsection{Knowledge Graphs and Structured Knowledge}

Knowledge graph approaches \citep{bordes2013translating,ji2022survey} represent information as structured triples, enabling precise querying and relationship traversal. GraphRAG \citep{edge2024local} extends this by constructing knowledge graphs from documents to improve retrieval over global queries. However, knowledge graphs require explicit entity extraction and relationship typing---a lossy transformation from natural-language knowledge that discards nuance, qualification, and context-dependent meaning.

ContextNest takes a different approach: documents remain in their authored form (Markdown), while structured metadata (frontmatter) and explicit references (\texttt{contextnest://} URIs) provide the queryable structure that knowledge graphs offer without requiring semantic decomposition.

\subsection{Version Control for Data and Documents}

Git \citep{torvalds2005git} provides content-addressable storage and cryptographic integrity for source code. DVC \citep{kuprieiev2021dvc} extends version control to data and ML artifacts. LakeFS \citep{treeverse2020lakefs} applies Git-like branching to data lakes. These systems version \emph{files} but do not model the semantic metadata (authorship, approval status, document type, cross-references) required for governed knowledge.

ContextNest's version history model is inspired by Git's content-addressable approach but operates at the document level with domain-specific semantics: each version entry records not just content changes but authorship, publication status, and a hash chain that enables independent integrity verification without a centralized server.

\subsection{Model Context Protocol}

The Model Context Protocol (MCP) \citep{anthropic2024mcp} standardizes how AI applications provide context to language models through a client-server architecture. MCP defines tool use, resource access, and prompt templates as primitives for model-context interaction. ContextNest's MCP server implementation exposes vault operations as MCP tools, and source nodes declaratively specify MCP tool calls for live data hydration.

\subsection{Data Provenance and Audit Trails}

Provenance tracking in databases \citep{buneman2001why,green2007provenance} and scientific workflows \citep{moreau2011open} has a rich history. The W3C PROV specification \citep{groth2013prov} defines a general provenance data model. In AI systems, provenance has received attention primarily in the context of training data documentation \citep{gebru2021datasheets} and model cards \citep{mitchell2019model}.

ContextNest extends provenance to the \emph{inference-time knowledge supply chain}: not which data trained the model, but which documents informed a specific AI interaction, who authored them, and whether they have been tampered with since.

\subsection{Trustworthy Agentic Systems and Cross-Organizational Standards}

A complementary literature has emerged around trustworthy agentic systems---the governance, observability, and evaluation requirements for AI systems that take operational actions across organizational boundaries~\citep{konsynski2026trustworthy}. This literature is itself the contemporary form of a longer Information Systems tradition on the progressive offloading of cognitive labor from humans to machines: from \emph{delegation} of perception and analysis to a community of intelligent agents~\citep{elofson1991delegation}, to dynamic \emph{cognitive apportionment} between human and machine processors~\citep{fjeldstad1986apportionment}, to formal \emph{cognitive reapportionment} of judgment and decision rights between human and system co-cognitors~\citep{sviokla1994reapportionment, konsynski2024cognitive}. Each stage of the lineage has identified the same enabling factor for the next: a substrate of grounded, attributable, auditable knowledge that the system can be trusted to consume. The control-plane decomposition we adopt in §\ref{sec:implementation} (policy, evaluation, telemetry, governance) is drawn from this literature. Adjacent standards relevant to cross-organizational agent interaction include the Agent Payments Protocol~\citep{ap22025}, Mastercard's Agent Pay framework for agentic commerce~\citep{mastercard2025agentpay}, ISO/IEC~42001 for AI management systems~\citep{iso42001}, the NIST AI Risk Management Framework~\citep{nist2023ai}, OpenTelemetry as a telemetry baseline~\citep{opentelemetry}, and the OWASP Top~10 for LLM applications as an adversarial baseline~\citep{owaspllm}. ContextNest provides the inference-time knowledge-governance substrate compatible with these standards: a portable specification of the data layer that their cross-organizational governance frameworks reference, and the artifact-level mechanism for the trustworthy-substrate requirement the IS lineage above has been articulating since the 1980s.

\subsection{Regulatory Context}

The EU AI Act \citep{euaiact2024} imposes transparency requirements on AI systems, including obligations to document the data used in AI decision-making. The NIST AI Risk Management Framework \citep{nist2023ai} identifies data governance as a core function. SOC~2 audits increasingly examine AI data provenance. These regulatory pressures create a practical need for auditable knowledge governance that current RAG architectures cannot satisfy.

\section{Problem Formalization}
\label{sec:problem}

We define \textbf{context governance} as the set of properties required for an AI system to consume external knowledge in a manner that is auditable, trustworthy, and compliant with organizational and regulatory requirements.

\begin{definition}[Context Governance Properties]
\label{def:governance}
Let a context system $\mathcal{S} = (D, V, C, A)$ consist of a set of documents $D$, a set of versions $V$ (where each $v \in V$ is associated with a unique document $d(v) \in D$ and an ordinal index $i(v) \in \mathbb{N}$), an append-only checkpoint log $C$, and an access log $A$. We say $\mathcal{S}$ is \textbf{governed} if it satisfies all six of the following properties.
\begin{enumerate}[nosep]
  \item \textbf{Provenance.} For every $v \in V$, $\mathcal{S}$ exposes a tuple $\pi(v) = (\mathrm{author}(v), \mathrm{created\_at}(v), \mathrm{edited\_at}(v))$ that is bound to $v$ and produced as part of any retrieval result.
  \item \textbf{Version identity.} For every document $d \in D$, the versions $\{v \in V : d(v) = d\}$ form a totally ordered sequence indexed by $\mathbb{N}$, and any access record $a \in A$ resolves to a unique $v$.
  \item \textbf{Integrity.} There exists an algorithm $\mathrm{Verify}: V \to \{\bot, \top\}$ such that any post-publication modification to $v$'s content, $\pi(v)$, or position in the sequence yields $\mathrm{Verify}(v) = \bot$ with overwhelming probability under standard cryptographic assumptions on the underlying hash function.
  \item \textbf{Deterministic selection.} The query function $Q : \mathrm{Selectors} \times \mathbb{N} \to 2^V$ is pure: for any selector $s$ and checkpoint number $n$, $Q(s, n)$ depends only on $s$, $n$, and the immutable history of $\mathcal{S}$ up to $n$---not on retrieval-time state, randomness, or implementation.
  \item \textbf{Traceability.} For every model output $o$ produced by an agent over $\mathcal{S}$, $A$ contains records $\{a_1, \dots, a_k\}$ such that each $a_i$ identifies a specific $v \in V$, and $\{v_i\}$ is exactly the context that informed $o$.
  \item \textbf{Temporal consistency.} For any $n \in \mathbb{N}$, $\mathcal{S}$ can reconstruct the set of versions current at checkpoint $n$ such that the reconstruction is bit-identical across invocations and across implementations conforming to the specification.
\end{enumerate}
\end{definition}

\begin{proposition}
\label{prop:rag-fails}
A standard RAG architecture---an embedding index over chunked documents with top-$k$ retrieval---does not provide Properties~1--6 of Definition~\ref{def:governance} as part of the retrieval abstraction. Production systems may supply these properties via additional infrastructure layered atop the retrieval mechanism, but such infrastructure is external to RAG and not the subject of the present claim.
\end{proposition}

\begin{proof}
We address each property by counterexample structural to the RAG abstraction.
\begin{enumerate}[nosep]
  \item \emph{Provenance.} The canonical RAG index stores $(\mathrm{embedding}, \mathrm{chunk\_id}) \to \mathrm{text}$ entries. Authorship metadata is not part of the index schema and is therefore not part of any retrieval result returned by the abstraction.
  \item \emph{Version identity.} On document update, the standard pipeline re-embeds the document and either overwrites the prior chunks or relies on chunk identity to detect duplicates; neither maintains a per-document version sequence. There is no $i(v)$.
  \item \emph{Integrity.} The vector index is mutable storage with no chained signature over chunks; modification of a stored chunk produces no detectable signal. There is no $\mathrm{Verify}$.
  \item \emph{Deterministic selection.} Approximate-nearest-neighbor structures (HNSW, IVF, ScaNN) depend on insertion order and hyperparameters; even exact $k$-NN can break ties non-deterministically. $Q$ is therefore not pure.
  \item \emph{Traceability.} RAG pipelines do not by default emit per-output access records that bind to a specific version; the caller may log retrieved chunks, but the system does not guarantee it, and chunks are not version-identified.
  \item \emph{Temporal consistency.} The index has no snapshot mechanism at the API level; recovering the index state at a prior point requires replaying writes from a backup, which is not a property of the RAG abstraction itself.
\end{enumerate}
Each negative result follows from the absence of the relevant mechanism in the RAG abstraction, not from a particular implementation choice. The claim of Proposition~\ref{prop:rag-fails} is therefore architectural: governance must be supplied by a layer the retrieval abstraction does not provide.
\end{proof}

The remainder of this paper presents ContextNest as a system that satisfies all six properties, and §\ref{sec:threat} specifies the adversaries against which Property~3 holds.

\subsection{Threat Model}
\label{sec:threat}

Property~3 (integrity) and the systems built atop it (checkpoints in §\ref{sec:checkpoints}, audit traces in §\ref{sec:injection}) make claims of the form ``modification is detectable.'' We make those claims precise by specifying the adversaries ContextNest defends against, the capabilities each is granted, and the mechanisms that defeat them.

\paragraph{Adversary 1 --- Silent content tamperer.} \emph{Capabilities:} read--write access to vault storage, including version histories. \emph{Goal:} modify the content of a published version $v$ without detection. \emph{Defense:} the per-document chain hash (§\ref{sec:integrity}). Modifying $v$'s content invalidates $v$'s content hash, which propagates to break every chain hash from $v$ onward.

\paragraph{Adversary 2 --- History rewriter.} \emph{Capabilities:} same as Adversary~1, plus the ability to rewrite version metadata (author, timestamp, version ordinal). \emph{Goal:} attribute a malicious edit to a different author, backdate it to predate review, or splice the version sequence. \emph{Defense:} the chain hash binds author, timestamp, and ordinal as part of its input. Rewriting any of them invalidates the chain at that point and all subsequent points.

\paragraph{Adversary 3 --- Checkpoint forger.} \emph{Capabilities:} read--write access to the checkpoint log. \emph{Goal:} fabricate or alter a checkpoint asserting that a particular knowledge state existed at time $t$. \emph{Defense:} checkpoint hashes (§\ref{sec:checkpoints}) bind to per-document chain hashes via cross-chain inclusion. A forged checkpoint either fails to match real chain hashes or requires colluding rewrites of every referenced document history (which Adversary~2's defense already covers).

\paragraph{Adversary 4 --- Stale-version inducer.} \emph{Capabilities:} unable to modify the vault, but can influence retrieval (e.g., through a man-in-the-middle on a federated reference, or a misconfigured client cache). \emph{Goal:} cause the agent to consume a version older than the latest published one. \emph{Defense:} checkpoint-pinned URIs, the resolver's default-to-latest-published rule, and the audit trace (§\ref{sec:injection}). The trace records the consumed version, exposing staleness post-hoc even if it is not prevented in real time. The empirical analogue of this adversary's behavior is exercised in §\ref{sec:stale-attack}.

\paragraph{Out of scope.} ContextNest does not defend against: (a)~prevention of tampering by an authorized writer---the system provides cryptographic \emph{evidence} of tampering, not access control; (b)~confidentiality of vault contents; (c)~availability attacks against vault storage; (d)~social attacks such as a steward approving misinformation. (a) and (d) are addressed by complementary access-control and review-workflow layers built atop the specification; (b) and (c) are concerns of the deployment environment. Agent identity attacks---in which a non-authorized principal impersonates an authorized one---are addressed by complementary attestation mechanisms scoped to future work (§\ref{sec:limitations}, §\ref{sec:future-work}).

\section{Document Model}
\label{sec:document}

\subsection{Structure}

A ContextNest document is a standard Markdown file (GitHub Flavored Markdown, version 0.29-gfm) with a YAML frontmatter header. The frontmatter carries structured metadata; the body carries authored knowledge in Markdown.

\begin{lstlisting}
---
title: "Document Title"
type: document
tags: ["#engineering", "#api"]
status: published
version: 3
author: author@example.com
created_at: 2024-01-15T10:30:00Z
checksum: "sha256:a1b2c3..."
---

# Document Title

Markdown body with [cross-references](contextnest://path/to/doc).
\end{lstlisting}

\subsection{Node Types}

ContextNest defines eight node types that classify the role of each document within the knowledge graph:

\begin{table}[h]
\centering
\begin{tabular}{@{}ll@{}}
\toprule
\textbf{Type} & \textbf{Semantics} \\
\midrule
\texttt{document} & General documentation, guides, overviews \\
\texttt{snippet} & Short, reusable text fragments \\
\texttt{glossary} & Term definitions and vocabulary \\
\texttt{persona} & AI agent behavior definitions \\
\texttt{prompt} & Prompt templates and instructions \\
\texttt{source} & Instructions for fetching live data (see Section~\ref{sec:sources}) \\
\texttt{tool} & Tool documentation and usage guides \\
\texttt{reference} & External references, links, citations \\
\bottomrule
\end{tabular}
\caption{ContextNest node types.}
\label{tab:nodetypes}
\end{table}

The type system enables selective querying (e.g., ``all glossary entries'') and differential treatment by consuming agents (e.g., treating \texttt{persona} nodes as behavioral constraints rather than factual knowledge).

\subsection{Document Status}

Documents carry a \texttt{status} field with two valid values: \texttt{draft} and \texttt{published}. Only published documents are eligible for context injection---the resolver excludes drafts from all query results. This provides a binary gate: work-in-progress knowledge cannot inadvertently reach AI agents.

\subsection{Vault Structure}

A ContextNest vault is a directory with a defined layout:

\begin{lstlisting}
vault/
  CONTEXT.md              # Vault identity and agent instructions
  context.yaml            # Auto-generated document graph index
  .context/
    config.yaml           # Vault configuration
  nodes/                  # Knowledge documents
    .versions/            # Version history
  sources/                # Source nodes for live data
    .versions/
  packs/                  # Context packs (saved queries)
\end{lstlisting}

\texttt{CONTEXT.md} serves as the vault's identity document---a vault-level system prompt that agents read before accessing individual documents. \texttt{context.yaml} is an auto-generated index containing the document registry, relationship graph, hub documents (ranked by inbound references), and external service dependencies.

\subsection{Stewardship Model}
\label{sec:stewardship}

The status field of §\ref{sec:document}.3 provides a binary publication gate, but it does not specify \emph{who} has the authority to flip a document from draft to published. ContextNest defines a hierarchical stewardship model that resolves this question deterministically. The model is composed of three structural elements---a scope hierarchy, a role lattice, and a separation-of-duties rule---governed by a vault-level mode switch that determines whether the enforcement layer is active. We describe the scope hierarchy and role lattice first, then the governance-mode switch (since solo and team vaults differ on whether the enforcement applies at all), and finally the separation-of-duties rule that the governed mode activates. Together, these realize the principle that humans remain accountable for the outcomes of AI agents that consume the vault---accountability is preserved at the artifact level through an auditable chain of authorship and approval.

\subsubsection{Scope Hierarchy}

A \emph{steward} is a principal---typically identified by email---who governs a subset of the vault. Each steward is bound at exactly one of four scopes. When the system needs to determine which steward governs a particular document (for review, approval, or read-access decisions), it resolves stewardship in priority order: \emph{first match wins}.

\begin{table}[h]
\centering
\small
\begin{tabular}{@{}clll@{}}
\toprule
\textbf{Priority} & \textbf{Scope} & \textbf{Binding} & \textbf{Example} \\
\midrule
1 & Document & A specific node identifier  & \texttt{nodes/pricing-policy} --- only this document \\
2 & Folder   & A path prefix               & \texttt{nodes/legal/} --- anything under \texttt{legal/} \\
3 & Tag      & A tag name (case-folded)    & \texttt{security} --- any document tagged \texttt{\#security} \\
4 & Vault    & The fallback scope          & The entire vault --- anything not covered above \\
\bottomrule
\end{tabular}
\caption{Stewardship scope hierarchy. Resolution order is by priority; ties at the same scope break by most-specific match.}
\label{tab:stewardship-scopes}
\end{table}

The resolution function $\sigma : D \to S$---assigning a steward $\sigma(d)$ to each document $d$---is total (the vault-scope fallback guarantees coverage) and deterministic (priority ordering plus a defined tie-break: longest path prefix, lexicographically smallest tag).

Stewardship assignments are themselves stored as a vault-level configuration (a \texttt{stewards.yaml} file at the vault root), making the assignment record portable, auditable, and version-controlled alongside the rest of the vault.

\subsubsection{Role Lattice}

Each steward holds one of three roles, partially ordered by capability:

\begin{table}[h]
\centering
\small
\begin{tabular}{@{}lccc@{}}
\toprule
\textbf{Role} & \textbf{Read} & \textbf{Edit} & \textbf{Approve / Reject} \\
\midrule
Viewer   & yes & no  & no  \\
Editor   & yes & yes & no  \\
Reviewer & yes & yes & yes \\
\bottomrule
\end{tabular}
\caption{Role lattice for stewards. Reviewer dominates Editor, which dominates Viewer.}
\label{tab:stewardship-roles}
\end{table}

A principal may hold different roles at different scopes; the effective role for any given document is the role bound at $\sigma(d)$.

\subsubsection{Governance Modes}
\label{sec:governance-modes}

A vault declares one of two governance modes via a top-level configuration flag. The mode determines whether the stewardship machinery of §\ref{sec:separation-of-duties} is active for this vault, so we state it before describing the enforcement rule itself.

\begin{itemize}[nosep]
  \item \textbf{Ungoverned} (default). New documents are auto-approved at creation and immediately AI-eligible. The stewardship machinery is inert. This mode is appropriate for personal vaults, prototypes, single-author research corpora, and any setting in which the author and the consumer of the vault are the same principal. The format-level guarantees of §\ref{sec:integrity}--§\ref{sec:checkpoints} (hash chains, checkpoints, audit trace) still apply; only the approval-workflow rule of §\ref{sec:separation-of-duties} is suspended.
  \item \textbf{Governed}. New documents enter as drafts. Approval requires a non-author Reviewer at the resolved scope (§\ref{sec:separation-of-duties}). Read access may additionally be gated by role.
\end{itemize}

Mode is a property of the vault, not of individual documents, so the governance posture of a vault is unambiguous from inspection of the configuration. A vault may be migrated from \emph{ungoverned} to \emph{governed} as the team or risk surface grows; the migration is a one-line configuration change, and the existing history---including pre-migration approvals recorded under the ungoverned default---is preserved verbatim under the chain-hash protections of §\ref{sec:integrity}. This migration pattern is the artifact-level realization of \emph{graduated autonomy}: cognitive authority is reapportioned to the system in proportion to the trust threshold the organization is willing to establish~\citep{konsynski2024cognitive}, and the vault posture follows.

\subsubsection{Separation of Duties}
\label{sec:separation-of-duties}

In \emph{governed} mode (§\ref{sec:governance-modes}), the model enforces a structural rule that no individual principal may unilaterally promote their own work into the published, AI-eligible state. (In \emph{ungoverned} mode the rule does not apply: a single author can author and publish in one step. The rule below is the governed-mode addition.)

\begin{proposition}[Authorial Separation]
\label{prop:authorial-separation}
In governed mode, for any version $v \in V$ with author $\alpha(v)$, the principal who effects the transition $\mathrm{status}(v) = \mathtt{draft} \to \mathtt{published}$ must satisfy $\mathrm{principal} \neq \alpha(v)$.
\end{proposition}

Equivalently: even a Reviewer-level steward cannot approve a version they themselves authored. The rule is enforced at the system boundary---at the API layer that mutates document status---not at the UI layer, so it is robust to client misbehavior. This is the standard control-theoretic separation of duties applied to AI knowledge supply: the editor of a version and its approver must be distinct principals.

\subsubsection{Relation to Property 5 (Traceability)}

The stewardship model strengthens Property~5 of Definition~\ref{def:governance}. The audit log $A$ records not only \emph{which version} informed each AI output but, transitively via the stewardship binding at the time of approval, \emph{which principal authorized that version's eligibility for AI consumption}. Approval events are first-class entries in the version history (§\ref{sec:integrity}) and inherit the chain-hash protections of §\ref{sec:integrity}---an attempt to retroactively rewrite an approval event invalidates the chain. The resulting record functionally constitutes an \emph{evidence bundle} for the agent's consumption of the artifact, as discussed in §\ref{sec:injection}.

\section{Selector Grammar}
\label{sec:selectors}

ContextNest defines a deterministic, set-algebraic query language for selecting documents. Unlike embedding-based retrieval, selectors produce identical results for identical queries---satisfying Property~4 (deterministic selection).

\subsection{Formal Grammar}

\begin{lstlisting}
selector := term (('|' term)*)
term     := factor ((('+' | ' ') factor)*)
factor   := atom | atom '-' atom | '(' selector ')'
atom     := tag | uri | pack_ref | type_filter
          | status_filter | transport_filter | server_filter
tag         := '#' IDENTIFIER
uri         := 'contextnest://' PATH ('@' INTEGER)? ('#' ANCHOR)?
pack_ref    := 'pack:' IDENTIFIER
type_filter := 'type:' NODE_TYPE
status_filter   := 'status:' STATUS
transport_filter := 'transport:' TRANSPORT
server_filter    := 'server:' IDENTIFIER
\end{lstlisting}

\subsection{Operator Semantics}

Operators perform set operations over the universe of published documents:

\begin{table}[h]
\centering
\begin{tabular}{@{}ccl@{}}
\toprule
\textbf{Operator} & \textbf{Name} & \textbf{Semantics} \\
\midrule
\texttt{+} (or space) & Intersection & Documents matching both operands \\
\texttt{|}             & Union        & Documents matching either operand \\
\texttt{-}             & Difference   & Documents matching left but not right \\
\texttt{( )}           & Grouping     & Precedence override \\
\bottomrule
\end{tabular}
\caption{Selector operators. Precedence (highest to lowest): \texttt{()} $>$ \texttt{+} $>$ \texttt{-} $>$ \texttt{|}.}
\label{tab:operators}
\end{table}

\subsection{Complexity}

Selector evaluation runs in $O(|S| \cdot \bar{p} + r)$, where $|S|$ is the number of atomic terms in the selector, $\bar{p}$ is the average size of a tag, type, or status posting list, and $r$ is the size of the result set. Set operators (\texttt{+}, \texttt{|}, \texttt{-}) are linear in the operand sizes. Posting lists are maintained at write time, so evaluation does not require a full vault scan. Pack expansion is one level of indirection---packs may not nest recursively---and so adds at most a constant factor.

\subsection{Context Packs}

Packs are named, stored selector expressions with optional agent instructions. They compose: a selector can reference a pack (\texttt{pack:sprint.standup}), and the pack's query is expanded inline during evaluation. Packs may include both static document nodes and source nodes; the resolver hydrates sources according to declared dependency ordering.

\section{Addressable Context Scheme}
\label{sec:uri}

ContextNest defines a URI scheme for stable, versionable references between documents:

\begin{table}[h]
\centering
\begin{tabular}{@{}ll@{}}
\toprule
\textbf{URI Pattern} & \textbf{Resolves To} \\
\midrule
\texttt{contextnest://path} & Latest published version \\
\texttt{contextnest://path@N} & Version at checkpoint $N$ \\
\texttt{contextnest://path\#anchor} & Section within document \\
\texttt{contextnest://tag/\{name\}} & All documents with tag \\
\texttt{contextnest://folder/} & All documents in folder \\
\texttt{contextnest://search/\{query\}} & Full-text search \\
\bottomrule
\end{tabular}
\caption{ContextNest URI patterns.}
\label{tab:uri}
\end{table}

\subsection{URI Grammar}
\label{sec:uri-grammar}

The URI patterns of Table~\ref{tab:uri} are defined by the following grammar (ABNF, RFC~5234). A URI's resolution class---direct, set, or delegated (§\ref{sec:resolution-classes})---is determined syntactically by which production it matches.

\begin{verbatim}
uri         = "contextnest://" [ authority "/" ] resource
authority   = namespace                  ; federated / scoped modes only
resource    = direct / set / delegated
direct      = path [ "@" checkpoint ] [ "#" anchor ]
set         = "tag/" tagname / path "/"  ; folder form = trailing slash
delegated   = "search/" query
path        = segment *( "/" segment )
segment     = 1*( unreserved / pct-encoded )
checkpoint  = 1*DIGIT                     ; checkpoint ordinal N
anchor      = segment                     ; intra-document section id
tagname     = segment
query       = 1*( unreserved / pct-encoded / "+" )
unreserved  = ALPHA / DIGIT / "-" / "." / "_" / "~"
pct-encoded = "%" HEXDIG HEXDIG
\end{verbatim}

\noindent Canonicalization (§\ref{sec:uri}, ``URI Canonicalization'') applies before matching: dot-segments are resolved, consecutive slashes rejected, percent-encoding normalized, the authority lowercased, and path segments treated as case-sensitive.

\subsection{Resolution Modes}

\textbf{Floating resolution} (no \texttt{@N}): Returns the latest published version. Appropriate for most cross-references where the author intends ``whatever the current version is.''

\textbf{Pinned resolution} (\texttt{@N}): Returns the document version recorded at checkpoint~$N$ (see Section~\ref{sec:checkpoints}). This provides a checkpoint-consistent view of the knowledge graph---essential for audit replay (Property~6).

\subsection{Direct, Set, and Delegated Resolution}
\label{sec:resolution-classes}

The URI patterns of Table~\ref{tab:uri} fall into three resolution classes, with different guarantees.

\emph{Direct addressing} ($\texttt{contextnest://path}$, $\texttt{path@N}$, $\texttt{path\#anchor}$) resolves to a specific document or section. The result is a single document (or section reference) determined entirely by the path, optional checkpoint, and optional anchor. Direct addressing supports Properties~1--6 of Definition~\ref{def:governance} without reservation.

\emph{Set addressing} ($\texttt{contextnest://tag/\{name\}}$, $\texttt{contextnest://folder/}$) resolves to the set of published documents matching the tag or folder predicate. Membership is determined by the frontmatter and the vault layout at the resolution checkpoint; the result is deterministic (Property~4) because the predicate is structural, not similarity-based. The selector grammar of §\ref{sec:selectors} is the algebra over these set-addressing primitives.

\emph{Delegated retrieval} ($\texttt{contextnest://search/\{query\}}$) is the integration surface for semantic and similarity-based retrieval. \textbf{The specification defines the URI; it does not define the retrieval algorithm.} A conforming engine delegates resolution of \texttt{search/} URIs to a configured retrieval back-end---a RAG pipeline, a hybrid sparse+dense index, a BM25 retriever, or any other system that conforms to a thin resolver interface. The delegation is explicit in the URI form, so a reader (human or agent) inspecting a query knows immediately which class it belongs to and which guarantees apply.

The delegation has two consequences worth stating plainly. First, \texttt{search/} URIs are not, in general, deterministic in the sense of Property~4: a dense retriever may surface different documents on repeated identical queries (the structural property measured in §\ref{sec:determinism}), and a vanilla sparse retriever can be deterministic but is still similarity-based, not structurally selected. The specification flags this by classification rather than by algorithm: any URI consumer that requires Property~4 should select from the direct or set-addressing classes. Second, the audit trace (§\ref{sec:injection}) records the resolved set returned by the delegated retriever, with the chain-hash of each returned document, even though the retrieval mechanism itself is external. The composition pattern of §\ref{sec:complementary} is implemented through this URI class: governed selection over the published, integrity-verified subset uses direct or set addressing; semantic retrieval over that governed subset uses \texttt{search/} URIs whose back-end indexes only the published, current versions.

\subsection{Namespace Federation}

Vaults may declare a namespace and federation mode, enabling cross-vault references. Three modes are supported: \emph{anonymous} (default, all URIs resolve locally), \emph{federated} (URIs with an authority component resolve via a registry to remote vaults), and \emph{scoped} (like federated, but restricted to an explicit allow-list). Federation enables organizational knowledge architectures where multiple teams maintain independent vaults that can reference each other while preserving independent governance.

\subsection{URI Canonicalization}

URIs are canonicalized before resolution: path segments are case-sensitive, dot segments are resolved, consecutive slashes are rejected, percent-encoding is normalized, and the authority component is lowercased. Two URIs differing only in non-canonical representation must resolve identically.

\section{Integrity Verification}
\label{sec:integrity}

ContextNest uses SHA-256 hash chains to detect tampering of version histories and checkpoint logs---satisfying Property~3 (integrity). The requirement for such tamper-evident storage---``append-only logging, hash chains, WORM storage''---is independently identified in the trustworthy-agentic-systems literature as a precondition for cross-organizational agent accountability~\citep{konsynski2026trustworthy}.

\subsection{Document Version Chains}

Each version entry in a document's history carries two hash fields:

\textbf{Content hash}: SHA-256 of the version's content (full snapshot for keyframes, unified diff for intermediate versions).

\textbf{Chain hash}: A cryptographic link to all previous entries:

\begin{equation}
\texttt{chain\_hash}[n] = \text{SHA-256}\left(
\begin{aligned}
&\texttt{chain\_hash}[n{-}1] \mathbin{\|} \texttt{":"} \mathbin{\|} \\
&\texttt{content\_hash}[n] \mathbin{\|} \texttt{":"} \mathbin{\|} \\
&\texttt{version}[n] \mathbin{\|} \texttt{":"} \mathbin{\|} \\
&\texttt{edited\_by}[n] \mathbin{\|} \texttt{":"} \mathbin{\|} \\
&\texttt{edited\_at}[n]
\end{aligned}
\right)
\end{equation}

For the genesis entry, $\texttt{chain\_hash}[n{-}1]$ is replaced by the sentinel string \texttt{contextnest:genesis:v1}.

The string concatenation above is performed over the canonical serialization of each field defined by RFC~8785 (JSON Canonicalization Scheme, JCS) \citep{rfc8785}, which fixes key ordering, number representation, whitespace, and Unicode normalization. This guarantees that two conforming implementations of \texttt{Verify} produce bit-identical chain hashes from equivalent histories.

\textbf{Property}: If any entry in the chain is modified---including its content, author attribution, timestamp, or position---all subsequent chain hashes become invalid, and verification detects the specific point of divergence. The construction follows Merkle's seminal hash-chaining technique \citep{merkle1987digital}, specialized to a per-document append-only log.

\subsection{Version Storage Model}

Version history uses a keyframe-plus-diff model for space efficiency:

\begin{itemize}[nosep]
  \item \textbf{Keyframe versions} (version~1 and every $k$-th version, default $k = 10$): stored as full Markdown snapshots.
  \item \textbf{Intermediate versions}: stored as unified diffs from the previous version.
  \item \textbf{Reconstruction}: Any version can be reconstructed by applying diffs forward from the nearest keyframe.
\end{itemize}

This mirrors techniques from video compression and database log-structured storage, balancing storage efficiency against reconstruction cost.

\section{Checkpoint System}
\label{sec:checkpoints}

Per-document versioning alone cannot guarantee cross-document consistency. When Document~B is republished, any document referencing it immediately resolves to the new version---potentially breaking assumptions made by the referring document's author. A \textbf{nest checkpoint} provides an atomic snapshot of the entire knowledge graph.

\subsection{Checkpoint Creation}

Each publication event (of any document) triggers a new checkpoint entry in the append-only checkpoint log. Each entry records: a monotonically increasing checkpoint number; the document that triggered the checkpoint; a complete map of every published document to its version number at that instant; a map of every published document to the chain hash of its recorded version (cross-chain binding); and a checkpoint hash chaining this entry to the previous checkpoint.

\subsection{Checkpoint Hash Construction}

\begin{equation}
\texttt{cp\_hash}[n] = \text{SHA-256}\left(
\begin{aligned}
&\texttt{cp\_hash}[n{-}1] \mathbin{\|} \texttt{":"} \mathbin{\|} \\
&\texttt{checkpoint}[n] \mathbin{\|} \texttt{":"} \mathbin{\|} \\
&\texttt{at}[n] \mathbin{\|} \texttt{":"} \mathbin{\|} \\
&\texttt{triggered\_by}[n] \mathbin{\|} \texttt{":"} \mathbin{\|} \\
&\texttt{canonical\_versions}[n] \mathbin{\|} \texttt{":"} \mathbin{\|} \\
&\texttt{canonical\_chain\_hashes}[n]
\end{aligned}
\right)
\end{equation}

Where \texttt{canonical\_versions} and \texttt{canonical\_chain\_hashes} are RFC~8785 (JCS) \citep{rfc8785} serializations of the version and chain-hash maps. We use JCS rather than an ad-hoc canonicalization because it is independently specified, has reference implementations across languages, and removes ambiguity about number representation and Unicode handling---all of which would otherwise cause cross-implementation hash divergence.

\textbf{Cross-chain binding}: The inclusion of \texttt{canonical\_chain\_hashes} cryptographically anchors each checkpoint to the per-document version chains. A verifier confirms that each chain hash in the checkpoint matches the corresponding entry in the document's history. A mismatch indicates that a document's history was rewritten after the checkpoint was created.

\subsection{Temporal Reconstruction}

Given checkpoint number~$N$, the system reconstructs the exact knowledge graph state by: (1)~loading the checkpoint entry for~$N$; (2)~for each document in \texttt{document\_versions}, reconstructing the recorded version from the document's history; (3)~verifying the chain hash of each reconstructed version against \texttt{document\_chain\_hashes}. This satisfies Property~6 (temporal consistency) and enables audit replay.

\paragraph{Complexity.} Reconstruction at checkpoint~$n$ requires $O(|D_n|)$ document loads plus diff replay bounded by the keyframe interval~$k$ per document, for total work $O(|D_n| \cdot k)$. Verification of a checkpoint is $O(|D_n|)$ hash comparisons. Both are linear in the size of the published vault at~$n$ and embarrassingly parallel across documents.

\subsection{Checkpoint History Rebuild}

If the checkpoint log is lost or corrupted, it can be deterministically rebuilt from the per-document histories by collecting all publication events (identified by \texttt{published\_at} timestamps), sorting chronologically, and replaying to reconstruct the checkpoint sequence. A rebuild from intact per-document histories produces an equivalent checkpoint log.

\section{Context Injection and Tracing}
\label{sec:injection}

\subsection{Injection Protocol}

Agents request context by selector query (Section~\ref{sec:selectors}) or URI (Section~\ref{sec:uri}). The resolver: (1)~evaluates the query against published documents; (2)~for source nodes in the result set, orders them by topological sort of the \texttt{depends\_on} graph; (3)~returns the resolved documents with their metadata and version information; (4)~logs the access for audit tracing.

The resolver returns documents; the agent decides whether and how to act on them. For source nodes, the resolver returns the instructions---the agent executes them.

\subsection{Audit Trace Schema}

Every context access produces a trace record containing: the URI of the accessed document, the version number consumed, the current checkpoint at time of access, the document author, the last edit timestamp, and the access timestamp.

For source node hydration (when the agent executes the described tool calls), additional fields record: tools called, server used, result hash (SHA-256 of hydrated content), cache status, and duration.

This satisfies Property~5 (traceability). A complete trace reads: ``The agent resolved \texttt{pack:sprint.standup}, read 3~static documents at checkpoint~12, hydrated \texttt{sources/current-sprint-tickets} via Jira MCP (cache miss, result hash \texttt{sha256:9f1b\ldots}), and generated a summary.''

\paragraph{The audit trace as evidence bundle.}
The trace records described above functionally constitute an \emph{evidence bundle} in the sense developed for cross-organizational agentic commerce~\citep{konsynski2026trustworthy, ap22025}: a durable, machine-verifiable record of what an agent consumed, on whose authority, at what point-in-time, and under what integrity guarantees. Under the principle that AI may automate actions but humans remain accountable for outcomes, this evidence bundle is the mechanism through which the human accountable for an agent's behavior can demonstrate the precise knowledge basis for that behavior. Section~\ref{sec:limitations} discusses the limits of this mechanism, including the agent-identity gap that the bundle does not yet close.

\subsection{Sub-document Selection for Large Documents}
\label{sec:chunking}

The selector grammar (§\ref{sec:selectors}) and the direct/set addressing classes (§\ref{sec:resolution-classes}) operate at the document level. For corpora dominated by small or medium documents---runbooks, standards entries, ADRs, policy paragraphs---this is the appropriate unit: each document is small enough that whole-document injection imposes negligible token cost. For corpora that include large authored documents---enterprise manuals, regulatory specifications, multi-section policy handbooks---whole-document injection is wasteful and, depending on the model's context budget, infeasible.

The specification supports sub-document selection through the \texttt{contextnest://path\#anchor} pattern of Table~\ref{tab:uri}. An anchor is a stable identifier emitted at section boundaries in the document body (Markdown heading IDs by default, with optional explicit \texttt{\{\#anchor-name\}} suffixes for stability across heading edits). The resolver, given a URI of the form \texttt{contextnest://\allowbreak policies/\allowbreak handbook\#\allowbreak vendor-onboarding}, returns the section bounded by the named anchor and its successor in the document, together with the same metadata block that a whole-document resolution would return (version, checkpoint, chain hash, steward).

\paragraph{Selector-anchored chunking.}
The recommended pattern for large documents is therefore not to split the source file into many short documents, but to author the long document as a single governed artifact and reference its sections via anchored URIs. This preserves the selector $\to$ document binding required by Property~4: the anchored URI is part of the deterministic-selection class (§\ref{sec:resolution-classes}), and the audit trace records both the document version and the consumed anchor range. Concretely: an author of a 50-section handbook produces one document with stable anchors per section; a context pack (§\ref{sec:selectors}) declares which anchors to include for which agent task; the agent receives only the named sections, with full provenance, integrity, and traceability over each section it actually consumed.

The pattern composes with the federation and delegated-retrieval surfaces. Anchored URIs may cross-reference sections of a remote vault's document under the federation modes of §\ref{sec:uri} (assuming the remote vault publishes the anchor namespace). And a delegated retriever (§\ref{sec:resolution-classes}) may return \texttt{search/} results at the anchor granularity rather than the whole-document granularity, provided the retriever indexes the document's sections explicitly; the audit trace records the resolved anchor URIs, not just the parent documents. The structural guarantees of the document model are preserved at any granularity for which a stable anchor exists.

The specification does not mandate a particular anchoring discipline (Markdown headings, explicit anchor tags, line ranges, or character offsets); any conforming engine that resolves the anchor space deterministically and includes the anchor identifier in the audit trace satisfies the specification.

\section{Source Nodes: Live Data Integration}
\label{sec:sources}

A source node is a ContextNest document whose body contains instructions for fetching live context from external services. Source nodes are first-class members of the knowledge graph---authored, versioned, governed, and integrity-verified like any other node.

\subsection{Declarative Specification}

Source nodes carry a \texttt{source} metadata block in frontmatter declaring the transport protocol (\texttt{mcp}, \texttt{rest}, \texttt{cli}, \texttt{function}), server identity, required tools, dependencies on other sources, and cache TTL. The Markdown body contains natural-language instructions for the agent: what calls to make, in what order, with what parameters, and how to interpret the results.

\textbf{Design principle}: Frontmatter carries what machines index. The body carries what agents execute.

\subsection{Dependency Resolution}

Source nodes may declare dependencies on other source nodes via \texttt{depends\_on}. The resolver computes a topological sort of the dependency graph and returns sources in hydration order. Circular dependencies are rejected at validation time.

\subsection{Result Lifecycle}

Hydrated results are session-scoped---they are not written to the vault and do not participate in the versioning or integrity mechanisms. Source nodes store \emph{instructions}, not \emph{results}. When a hydrated result is reviewed by a human and promoted to a durable record, it is authored as a standard document node with \texttt{derived\_from} referencing the originating sources.

\subsection{The Staged Source-Node Lifecycle State}
\label{sec:staged-state}

The session-scoped lifecycle of §\ref{sec:sources}.3 keeps the durable vault free of unreviewed external content, but it leaves an audit gap: the trace records the result hash of each hydration and the metadata of the originating source, yet the hydrated content itself is not retained. For agents whose actions can be reverified against the live source's current state, this is the right tradeoff. For agents whose actions may need to be audited against the external state \emph{as it was at the time of consumption}---a state that may have changed irreversibly between consumption and audit---the absent content is exactly the missing evidence.

To close this gap without inverting the storage model, the specification introduces a third lifecycle position---the \textbf{staged} state---between session-scoped hydration and durable publication. The staged state is functionally analogous to a staging area in a version-control system: a captured, attributable, integrity-verified record that exists in the vault, is excluded from agent-eligible context by default, and is either promoted into the durable lifecycle (draft, then published) by a steward or garbage-collected after a configurable retention window.

\paragraph{Capture content.}
When a source node is hydrated and the vault is configured for staging, the engine writes the hydrated result into a staged entry alongside its existing audit-trace record. The staged entry carries the same metadata block as a published document version: the originating source URI, the hydration timestamp, the result hash, the principal who triggered the hydration, and the cache and dependency context. Additionally, an identifier is stored for the session that triggered the hydration.

\paragraph{Capture policy.}
Capture is unconditional with respect to the result content itself. The decision to stage a given result is governed by a policy that is configured at the vault level and applies at the source level. A given source may stage by default, never stage, or require an explicit \texttt{stage:} pragma in the source node's frontmatter.

\paragraph{Integrity construction.}
A staged entry carries a content hash over the canonical serialization of the hydrated payload (RFC~8785 JCS, consistent with §\ref{sec:integrity}) and participates in a per-source-node chain hash analogous to the per-document chain of §\ref{sec:integrity}. The chain binds: the previous staged-entry hash for the same source, the content hash, the originating source URI at its consumed version, the hydration timestamp, and the principal. Modifying a staged entry's content, metadata, or position in the source's staging sequence invalidates the chain at that point and all subsequent points, mirroring the integrity guarantees that apply to published documents. Checkpoint-level cross-chain binding (§\ref{sec:checkpoints}) optionally includes the staged tails of the configured sources, anchoring the staged stream into the vault-wide checkpoint sequence.

\paragraph{Selector visibility.}
Using the staged entry's session identifier, staged entries are included from selector queries with matching session identifiers and by default excluded from all other selector queries. The status predicate of §\ref{sec:document}.3 admits a third value, \texttt{staged}, that a selector may explicitly opt into when an agent or workflow has a legitimate need to read the staged content directly---for example, a stewardship workflow that reviews staged candidates for promotion. The default-excluded behavior preserves the §\ref{sec:document}.3 invariant: nothing that has not passed publication review can be consumed by an unguarded agent. The opt-in path (selectors that explicitly write \texttt{status:staged}) makes review-workflow tools first-class participants in the selector algebra rather than a side channel.

\paragraph{Promotion.}
A staged entry is promoted into the durable lifecycle by a steward at the resolved scope (§\ref{sec:stewardship}). Promotion writes a new \emph{document} node whose body is the staged entry's content, whose frontmatter records the staged-entry origin (a new \texttt{promoted\_from} field analogous to the existing \texttt{derived\_from}), and whose initial status follows the vault's governance mode: \texttt{draft} in \emph{governed} mode (subject to a non-author reviewer per §\ref{sec:separation-of-duties}), \texttt{published} in \emph{ungoverned} mode. The promotion event is itself an entry in the audit trace and inherits the chain-hash protections of §\ref{sec:integrity}. The staged entry is retained as a back-reference target so that the promoted document's \texttt{promoted\_from} URI continues to resolve.

\paragraph{Garbage collection.}
Staged entries that are neither promoted nor explicitly pinned within a configurable retention window are garbage-collected. The collection event is an entry in the audit trace; the collected entry's chain-hash leaf is retained as a tombstone so that the staging chain remains verifiable across the GC boundary. The retention window is a per-vault configuration, with optional per-source overrides for sources whose external state is unusually short-lived (e.g.\ live market data) or unusually durable (e.g.\ regulatory filings). Pinning is a steward-only operation that suspends GC for a named staged entry pending an explicit promotion or discard decision.

\paragraph{Configuration.}
Two vault-level flags govern the staged-state machinery: \texttt{staging.enabled} (default \texttt{false} in \emph{ungoverned} mode, default \texttt{true} in \emph{governed} mode) and \texttt{staging.retention} (default 30 days, configurable per source). A source node may override the vault default by declaring \texttt{stage: never}, \texttt{stage: default}, or \texttt{stage: always} in its frontmatter. Sources that declare \texttt{stage: never} retain the original session-scoped lifecycle of §\ref{sec:sources}.3 verbatim.

\paragraph{How the staged state closes the audit gap.}
With staging enabled, the audit trace of a source-node hydration now resolves to a durable, integrity-verified record of the hydrated content---not merely the hash of content that no longer exists. An auditor examining an agent action six months after the fact can reconstruct the exact external evidence the agent consumed, even if the external service has rotated, the original API endpoint has been deprecated, or the source page has been edited. The trade against permanent retention is bounded by the retention window: ephemeral content does not accumulate indefinitely in the vault, and the steward path is the only mechanism that elevates ephemeral content into durable knowledge. The staged state is thus the artifact-level realization of the principle that \emph{audit completeness is a property of the system, not of the underlying world}: the system retains what it needs to retain to reconstruct its own behavior, on a clock that the organization controls.

\section{Reference Implementation and Empirical Validation}
\label{sec:implementation}

The reference implementation of the ContextNest specification is developed and maintained by PromptOwl, LLC\ (the first author's affiliation), and released as three open-source packages described below. The architecture exposed by these packages may be usefully described through a four-plane decomposition---a \emph{policy plane} (identity, permissions, policy-as-code), an \emph{evaluation plane} (golden cases, rubrics, regression tests), a \emph{telemetry plane} (traces, logs, metrics, tamper-evident storage), and a \emph{governance plane} (RACI, evidence bundles, change management)---a decomposition that the second author has independently surfaced in the trustworthy-agentic-systems literature~\citep{konsynski2026trustworthy}. We adopt the decomposition here for expository convenience; the ContextNest specification itself is architecture-independent, and alternative decompositions of a conforming implementation are equally valid. Under the four-plane reading, the reference implementation realizes the policy plane (through status and stewardship, §\ref{sec:stewardship}), the telemetry plane (through hash chains and audit traces, §\ref{sec:integrity}--§\ref{sec:injection}), and the artifact tier of the governance plane (through stewardship records and the evidence-bundle structure of the audit trace, §\ref{sec:injection}). The evaluation plane is exercised through the experimental program of §\ref{sec:experiments}; first results are reported in §\ref{sec:first-results}--§\ref{sec:stale-attack} below.

The three packages of the reference implementation are as follows.

\textbf{Context Engine} (\texttt{@promptowl/contextnest-engine}): Core library implementing document parsing, storage abstraction, selector evaluation, version management, checkpoint management, integrity verification, and context injection with tracing. Authored and maintained by PromptOwl, LLC; released under AGPL-3.0.

\textbf{CLI} (\texttt{@promptowl/contextnest-cli}): Command-line tool providing 19~commands for vault initialization, document management, querying, versioning, and integrity verification. Includes starter recipes for common vault configurations. Authored and maintained by PromptOwl, LLC; released under AGPL-3.0.

\textbf{MCP Server} (\texttt{@promptowl/contextnest-mcp-server}): A Model Context Protocol server exposing vault operations as tools for AI agents. The tool surface is summarized in Table~\ref{tab:mcp-tools}. Authored and maintained by PromptOwl, LLC; released under AGPL-3.0.

PromptOwl, LLC\ also develops additional software that consumes ContextNest vaults at runtime, including an agentic orchestration runner and a desktop client. That software is out of scope for this paper.

\paragraph{License rationale.}
The license split between specification and implementation is deliberate. All three reference-implementation packages---the Context Engine, the CLI, and the MCP Server---are released under AGPL-3.0 as defensive copyleft: the goal is to keep derivatives of the reference \emph{implementation} open, so that improvements to the governance machinery flow back to the ecosystem rather than being absorbed into proprietary forks. The network-copyleft form (AGPL rather than ordinary GPL) is chosen deliberately so that the source-availability obligation is triggered even when the software is offered as a hosted service rather than distributed. The choice is not motivated by any GPL-licensed dependency in the codebase; it is a posture, not a compliance obligation. The specification itself (the \texttt{spec/} directory of the canonical repository) is released under Apache-2.0: the specification is intended to admit any conforming implementation, including proprietary ones, and a widely-adopted permissive license is the standard mechanism for that. Together the two licenses encode a policy: the spec is permissively open so that anyone may implement it, while the reference implementation's improvements stay open by construction.

\begin{table}[h]
\centering
\small
\begin{tabular}{@{}llp{0.55\textwidth}@{}}
\toprule
\textbf{Tool} & \textbf{Mutation?} & \textbf{Purpose} \\
\midrule
\texttt{context\_init}            & no  & Load vault \texttt{CONTEXT.md} (vault-level operating instructions). \\
\texttt{context\_overview}        & no  & Vault map: total nodes, types, tags, title-and-snippet per node. \\
\texttt{context\_search}          & no  & Full-text keyword search across content, titles, tags, metadata. \\
\texttt{context\_resolve}         & no  & Resolve a selector or \texttt{contextnest://} URI to matching nodes. \\
\texttt{context\_read}            & no  & Read a single node by id, with full body and frontmatter. \\
\texttt{context\_neighbors}       & no  & Traverse the reference graph from a node---inbound and outbound. \\
\texttt{context\_pack}            & no  & Resolve a stored pack by id, hydrating included sources. \\
\texttt{context\_diff}            & no  & Compare two versions of a node. \\
\texttt{context\_history}         & no  & Return the version history of a node, with chain hashes. \\
\texttt{context\_verify}          & no  & Verify the chain-hash integrity of a node or the whole vault. \\
\texttt{context\_publish}         & yes & Promote a draft to published (subject to stewardship, §\ref{sec:stewardship}). \\
\texttt{context\_create}          & yes & Author a new draft node. \\
\texttt{context\_update}          & yes & Edit an existing draft node. \\
\texttt{context\_assign\_steward} & yes & Bind a principal to a scope at a role (governed mode only). \\
\bottomrule
\end{tabular}
\caption{MCP tool surface exposed by the reference server. Mutation tools are subject to stewardship checks (§\ref{sec:stewardship}) and contribute to the audit trail (§\ref{sec:injection}).}
\label{tab:mcp-tools}
\end{table}

Each invocation of any tool produces an entry in the audit log $A$ described abstractly in §\ref{sec:injection}: the resolved node identifier, the consumed version, the current checkpoint, the principal, and (for source nodes) the hydration record. The audit log is itself a ContextNest artifact and inherits the chain-hash integrity protections of §\ref{sec:integrity}. The implementation includes a regression test suite that exercises the engine, the CLI, and the MCP server; passing tests are evidence of internal consistency, not of empirical performance. The specification, engine, CLI, and MCP server are available at \url{https://github.com/PromptOwl/context-nest}; reproduction details for the empirical work below are given in §\ref{sec:artifacts}.

\subsection{First Empirical Results}
\label{sec:first-results}

The experiments below are not intended to show that deterministic selection \emph{replaces} semantic retrieval. They isolate a class of failures---context that is textually relevant but organizationally obsolete, and retrieval that varies run-to-run---that retrieval relevance alone is not designed to resolve. We present them as first validation of a specification, not as a general retrieval benchmark.

We report a first run of experiment E1 (token cost: governed selection vs.\ retrieval-augmented baselines, §\ref{sec:experiments}) against a 10-query fixture suite drawn from a synthesized vault of runbooks, ADRs, and standards documents. Two retrieval conditions are compared: the deterministic selector grammar of §\ref{sec:selectors} (using \texttt{ctx resolve}) and BM25 sparse retrieval at $k{=}3$ over the same corpus. Both conditions use Claude Sonnet~4.6 as the answer model and Claude Opus~4.7 as an LLM-judge that grades each answer PASS/FAIL against a rubric of required facts.

\begin{table}[h]
\centering
\begin{tabular}{lrrr}
\toprule
Method & Avg.\ input tokens & Avg.\ output tokens & Pass rate \\
\midrule
Selector (\texttt{ctx resolve}) & 217 & 72 & 0.80 \\
BM25 ($k{=}3$) & 644 & 70 & 0.90 \\
\bottomrule
\end{tabular}
\caption{First E1 run, 10-query fixture suite. The selector achieves a ${\sim}3\times$ reduction in input tokens; BM25 wins by one query on surface pass-rate.}
\label{tab:e1-first-results}
\end{table}

We treat this as preliminary evidence, not a benchmark. The suite is small (10 queries vs.\ the 50 planned in §\ref{sec:experiments}); only two of the three E1 conditions (selector, BM25) were exercised; and the LLM-judge methodology is not yet calibrated against human inter-rater agreement. The headline finding---a ${\sim}3\times$ reduction in input tokens at comparable pass rate---is consistent with the central claim of E1 but warrants the full 50-query Pareto curve before being treated as confirmatory. We emphasize that this run measures \emph{token cost on a clean corpus}: every fact has a single current version, so retrieval quality is comparable across methods and BM25's one-query edge on surface pass-rate is not a meaningful difference. E1 is deliberately \emph{not} a test of governance. The governance question---what happens when the corpus also contains content that must \emph{not} be consumed---is isolated separately in §\ref{sec:stale-attack}, and that is where the distinction becomes decisive (Table~\ref{tab:stale-attack}). The two experiments answer different questions: E1 asks what governed selection \emph{costs}; the stale-version attack asks what it \emph{prevents}.

\subsection{Stale-Version Attack}
\label{sec:stale-attack}

The selector grammar of §\ref{sec:selectors} surfaces only documents in published state. A retrieval system that indexes the storage layer indiscriminately---including the keyframe-plus-diff history of §\ref{sec:integrity}---may surface superseded versions alongside current ones. We tested whether this distinction produces a measurable behavioral difference, designing a sharpened variant of the adversarial-poisoning experiment described in §\ref{sec:experiments} (E5).

\subsubsection*{Setup}

We extended the fixture vault of §\ref{sec:first-results} by authoring stale ``$v2$ archived'' entries for six published documents (three runbooks, two standards, one ADR). Each stale entry contradicts the current published version on five specific facts (numeric thresholds, named tools or channels, decision outcomes, SLAs). We then authored a 30-query suite in which each query targets one such fact, with the rubric grounded in the \emph{current} (correct) answer.

Three retrieval conditions were compared, all using Claude Sonnet~4.6 as the answer model and Claude Opus~4.7 as judge:
\begin{itemize}[nosep]
\item \textbf{Selector} (\texttt{ctx resolve}): tag-and-type predicate over the published vault. By construction, only current published documents are returned.
\item \textbf{BM25 (leaky)}: sparse retrieval at $k{=}3$ over a corpus that includes both current documents and the \texttt{.versions/} history. This models a vector or sparse-text pipeline that indexes the raw storage layer without filtering for publication state.
\item \textbf{BM25 (clean)}: sparse retrieval at $k{=}3$ over a corpus restricted to current published documents only. This models a properly configured production RAG pipeline that excludes version history.
\end{itemize}

\subsubsection*{Results}

Results are summarized in Table~\ref{tab:stale-attack} and visualized in Figure~\ref{fig:stale-attack}.

\begin{table}[h]
\centering
\begin{tabular}{lrr}
\toprule
Method & Avg.\ input tokens & Pass rate \\
\midrule
Selector (\texttt{ctx resolve}) & 215 & 0.97 \\
BM25 leaky (indexes \texttt{.versions/}) & 655 & 0.93 \\
BM25 clean (published-only corpus) & 725 & 0.90 \\
\bottomrule
\end{tabular}
\caption{Stale-version attack, 30-query suite, three retrieval conditions. The selector strictly Pareto-dominates both BM25 conditions: higher pass rate at lower input-token cost.}
\label{tab:stale-attack}
\end{table}

\begin{figure}[h]
\centering
\includegraphics[width=0.95\linewidth]{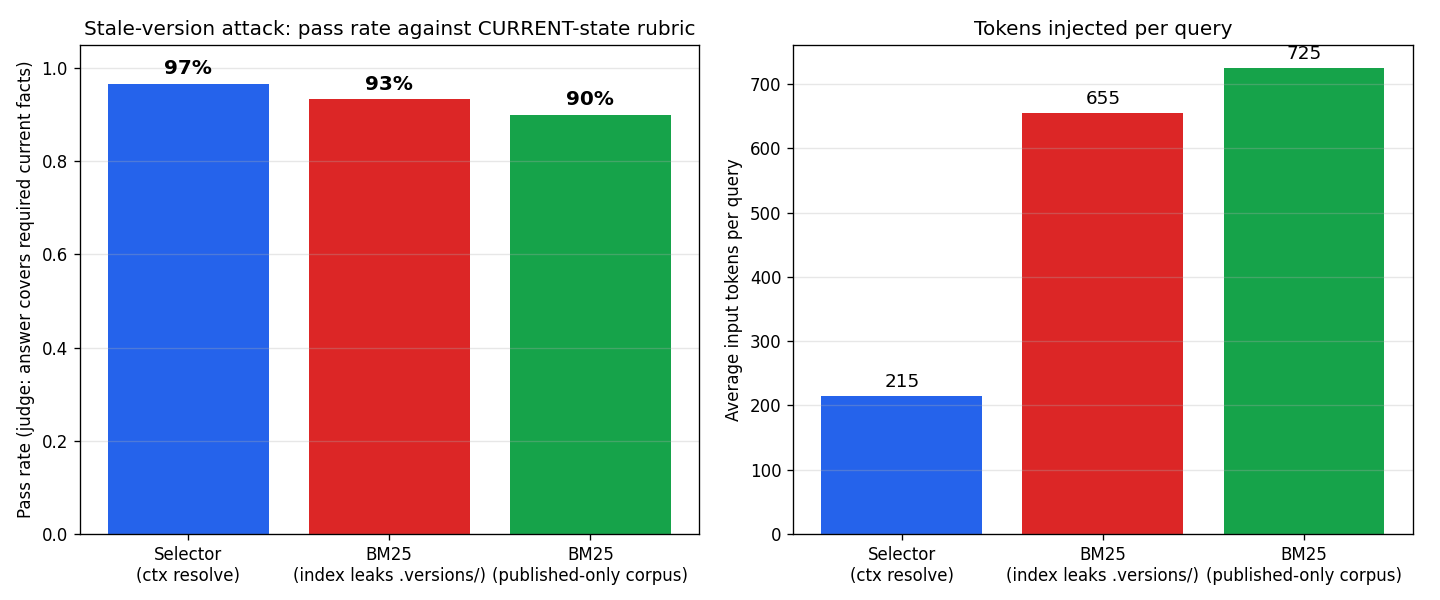}
\caption{Stale-version attack results. Left: pass rate against the current-state rubric. Right: average input tokens injected per query. Selector strictly Pareto-dominates both BM25 conditions on both axes.}
\label{fig:stale-attack}
\end{figure}

\subsubsection*{Two distinct failure modes for BM25}

The 1--3 query gap between the selector and the BM25 conditions is attributable to two qualitatively distinct failure modes, both demonstrated in this run.

\paragraph{Failure mode 1: stale-version poisoning.}
Query~s19 asks ``What error format do APIs use?'' The current published version of the API design standard specifies RFC~7807 problem-details~\citep{rfc7807}. The archived version specifies a custom \texttt{code}/\texttt{message} schema. The selector retrieves only the current version and produces an answer citing RFC~7807, marked PASS. \textbf{BM25 leaky retrieves the archived version} (along with two unrelated documents) and produces a confidently-stated answer enumerating the custom \texttt{code}/\texttt{message} schema---the superseded answer---marked FAIL. This is the failure mode the present paper argues most strongly against: an AI system reporting a confident, plausible answer that is grounded in evidence the organization has already retired.

\paragraph{Failure mode 2: retrieval miss.}
Query~s11 asks ``What defines a SEV1 incident?'' The selector tag predicate \texttt{\#runbook \#incident} returns the incident-response runbook by construction, and the model answers correctly. Both BM25 conditions instead surface the gRPC ADR, the onboarding guide, and the coding-conventions standard---none of which contain the requested information---and the model correctly reports that the context lacks the answer. This failure mode is independent of the version-leakage problem: it is a structural property of similarity-based retrieval against documents whose lexical overlap with the query is low even when the topical fit is exact. The selector grammar resolves this case by design because the relationship between query and document is expressed as typed tags, not as inferred similarity.

\subsubsection*{Discussion}

The selector achieves a strictly Pareto-dominant point: higher pass rate at lower token cost than either BM25 condition. The result is consistent with the central thesis of this paper---that governed selection over a published-state vault produces measurably different agent behavior than similarity-based retrieval over the same underlying corpus---and adds an empirical lower bound on the magnitude of the effect for tag-and-type queries against a small, well-structured vault.

We caution against over-reading the headline numbers. The 30-query suite is intentionally adversarial: every query targets a fact where stale and current versions contradict. The vault is synthesized and homogeneous in style. Only sparse retrieval was tested; the dense embedding baseline of E1 is not yet evaluated. The judge is a single LLM and inter-rater calibration is not yet reported. We treat the result as a demonstration of the failure modes the specification is designed to prevent, not as a benchmark of relative quality across realistic enterprise workloads. The full E1 grid (three retrieval conditions $\times$ three values of $k$, on the 50-query suite of §\ref{sec:experiments}), the determinism experiment (E2), and the multi-hop QA evaluation (E6) remain in progress.

\subsection{Determinism of Retrieval}
\label{sec:determinism}

Property~4 of Definition~\ref{def:governance} asserts that selector evaluation is a pure function of the selector expression, the checkpoint number, and the immutable vault history. This subsection reports a controlled measurement of that property against two retrieval baselines whose determinism is not guaranteed by their abstraction. The experiment corresponds to E2 in the program of §\ref{sec:experiments}.

\subsubsection*{Setup}

We synthesized a $1{,}060$-document corpus by replicating ten document templates (deploy-rollback runbooks, database-migration runbooks, incident-response runbooks, API-design standards, security-review standards, internal-protocol ADRs, architecture overviews, onboarding guides, monitoring runbooks, disaster-recovery runbooks) across $106$ service variants (payments-service, auth-service, billing-service, search-service, $\dots$). Each document carries a service-specific tag (e.g.\ \texttt{\#payments-service}) and topic tags (e.g.\ \texttt{\#runbook \#deploy \#ops}). Threshold values, owners, tools, and SLAs were randomized across documents from a fixed RNG seed, yielding a realistically structured but lexically heterogeneous corpus.

We then authored a 50-query suite, each query probing a specific fact in a specific service's documentation (e.g.\ ``What error rate threshold triggers a deploy rollback for the payments-service?''). For each query, we executed each retrieval method 20~times and recorded the set of retrieved document identifiers in each rep. We computed the mean pairwise Jaccard similarity across the 20~retrieved sets per (query, method) pair. A query is \emph{perfectly deterministic} under a method when its mean Jaccard is~1.0; \emph{non-deterministic} when $<\!1.0$.

The three retrieval methods compared were:
\begin{itemize}[nosep]
\item \textbf{Selector} (\texttt{ctx resolve}): tag-and-type predicate evaluation. The selector for query $q$ combines the service-specific and topic tags, returning the unique published document that matches (e.g.\ \texttt{\#runbook \#deploy \#payments-service}).
\item \textbf{BM25}: sparse top-$k$ retrieval ($k{=}3$) over the same corpus.
\item \textbf{Dense + HNSW}: top-$k$ retrieval ($k{=}3$) over \texttt{bge-small-en-v1.5} embeddings indexed with FAISS HNSW (\texttt{M=16}, \texttt{efConstruction=40}, \texttt{efSearch=4}). To stress two realistic sources of production non-determinism, we additionally rebuilt the HNSW index every 5~reps with a shuffled insertion order (exposing the well-documented insertion-order sensitivity of HNSW) and rotated which of the resulting variant indices each rep used.
\end{itemize}

(An initial pilot of this experiment on the 22-document fixture of §\ref{sec:first-results} produced a 17\% non-determinism rate for the dense+HNSW baseline. We scaled the corpus to $1{,}060$ documents because the rate of HNSW-related non-determinism is known to grow with the size and density of the embedding space.)

\subsubsection*{Results}

Results are summarized in Table~\ref{tab:determinism} and visualized in Figure~\ref{fig:determinism}.

\begin{table}[h]
\centering
\begin{tabular}{lrrrr}
\toprule
Method & Mean Jaccard & Min Jaccard & Perfectly det.\ queries & Non-det.\ queries \\
\midrule
Selector (\texttt{ctx resolve}) & \textbf{1.000} & 1.000 & \textbf{50 / 50} & 0 \\
BM25 ($k{=}3$) & \textbf{1.000} & 1.000 & \textbf{50 / 50} & 0 \\
Dense + HNSW ($\mathrm{efSearch}{=}4$) & 0.611 & 0.210 & 10 / 50 & \textbf{40 / 50} \\
\bottomrule
\end{tabular}
\caption{Retrieval-determinism results on the $1{,}060$-document synthesized corpus, 50 queries $\times$ 20 reps per method. The selector and BM25 baselines were perfectly deterministic on every query. The dense + HNSW baseline was non-deterministic on $40$ of $50$ queries (80\%); on the worst-affected query, repeated identical queries returned retrieved-document sets that overlapped only $21.0\%$ on average across the 20~reps.}
\label{tab:determinism}
\end{table}

\begin{figure}[h]
\centering
\includegraphics[width=0.95\linewidth]{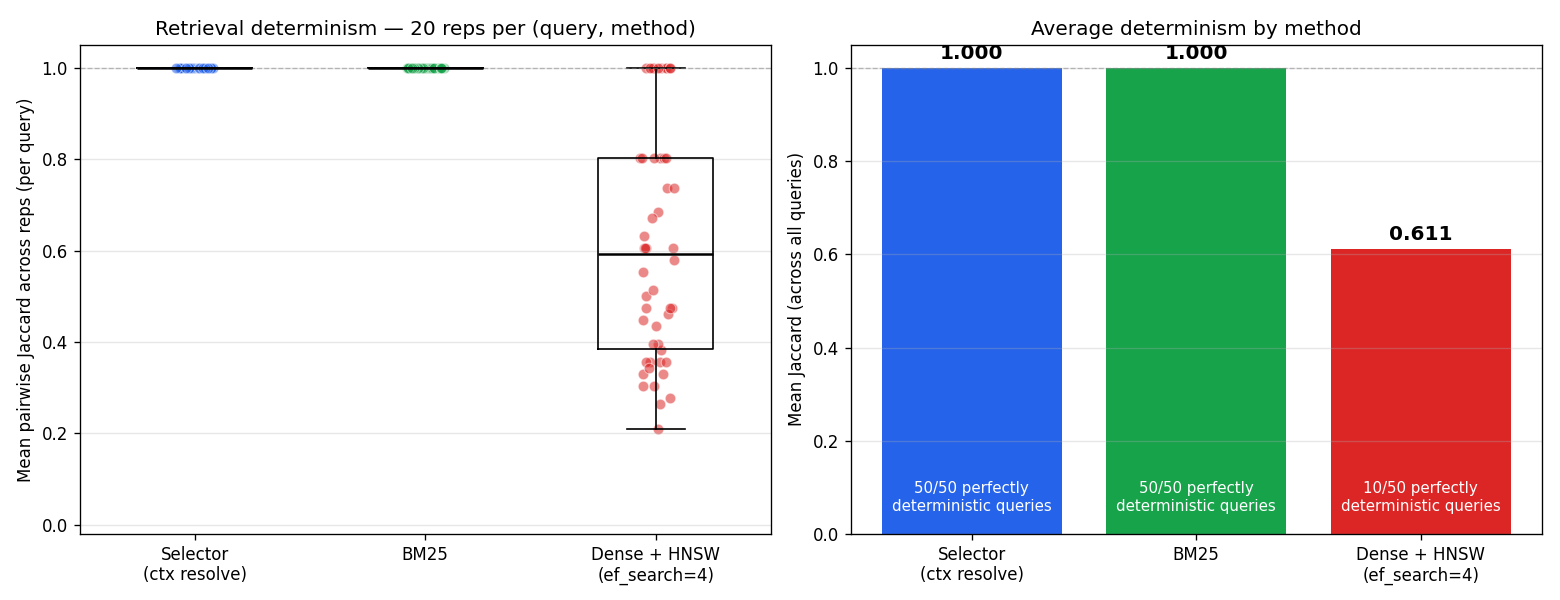}
\caption{Retrieval determinism across 50~queries with 20 repetitions per (query, method) pair on a $1{,}060$-document synthesized corpus. Left: per-query mean pairwise Jaccard; each dot is one query. Selector and BM25 yield a Jaccard of 1.0 on every query (all dots stacked at $y{=}1.0$). The dense+HNSW baseline's per-query Jaccard scores span the interval $[0.21, 1.0]$ with a median near $0.6$. Right: mean Jaccard across all queries, with the count of perfectly-deterministic queries annotated.}
\label{fig:determinism}
\end{figure}

\subsubsection*{Discussion}

The selector and BM25 baselines pass the determinism test trivially: both are deterministic algorithms applied to fixed inputs, and the test confirms this empirically on every query. The result of architectural interest is the dense + HNSW baseline. With realistic production parameters (low \texttt{efSearch}, insertion-order variance across index rebuilds) and a corpus large enough to exercise the embedding space, the dense baseline fails determinism on $80\%$ of queries. The mean Jaccard across queries is $0.611$, meaning that on an average query, repeated identical inputs return retrieved-document sets that share only $61\%$ of their members. On the worst-affected query, that overlap drops to $21\%$: nearly $4$ of every $5$ documents the baseline retrieves change from one identical-query execution to the next.

The relevance to context governance is direct. An agent that consumes context from a non-deterministic retrieval pipeline cannot reliably reproduce the inputs that justified a prior output. Audit, regulatory replay, incident investigation, and dispute resolution all require that the knowledge basis of an agent's action be reconstructible bit-identically after the fact. Property~4 (deterministic selection) is the structural precondition for that reconstructibility. A pipeline that fails the property on $80\%$ of its queries cannot supply traceability under any post-hoc inspection regime.

\paragraph{Limits of this measurement.}
The synthesized corpus is structured: 10 templates $\times$ 106 service variants. Production enterprise corpora are typically larger, less templated, and exhibit additional non-determinism factors not exercised here (GPU batching, sharded indices, distributed query coordination, periodic re-indexing). The corpus we used should be read as a controlled environment in which the structural property (selector determinism vs.\ dense non-determinism) becomes visible; production-scale magnitudes are expected to be at least as large. We treat this as a demonstration of the structural failure mode, not as a benchmark of relative quality. The full E2 grid of §\ref{sec:experiments} (100 queries $\times$ 100 reps $\times$ three retrieval conditions including hybrid sparse+dense) remains in progress.

\subsection{Artifact Availability}
\label{sec:artifacts}

The ContextNest specification, all three reference-implementation packages, the experimental harness, the fixture vault, and the rubrics used in §\ref{sec:first-results} and §\ref{sec:stale-attack} are openly available. Table~\ref{tab:artifacts} summarizes the artifact inventory; the canonical repository is hosted at \url{https://github.com/PromptOwl/context-nest}.

\begin{table}[h]
\centering
\small
\begin{tabular}{lll}
\toprule
Artifact & License & Path \\
\midrule
Specification & Apache-2.0 & \texttt{spec/} \\
\texttt{@promptowl/contextnest-engine} & AGPL-3.0 & \texttt{packages/engine/} \\
\texttt{@promptowl/contextnest-cli} & AGPL-3.0 & \texttt{packages/cli/} \\
\texttt{@promptowl/contextnest-mcp-server} & AGPL-3.0 & \texttt{packages/mcp-server/} \\
Experimental harness (Dockerized) & Apache-2.0 & \texttt{contextnest-eval/} \\
Fixture vault (11 published documents) & CC-BY-4.0 & \texttt{contextnest-eval/vaults/} \\
Query suites (E1, stale-attack) & CC-BY-4.0 & \texttt{contextnest-eval/queries*.yaml} \\
Stale-version archive content & CC-BY-4.0 & \texttt{contextnest-eval/vaults/nodes/**/.versions/} \\
\bottomrule
\end{tabular}
\caption{Artifact inventory.}
\label{tab:artifacts}
\end{table}

\paragraph{Reproducing the empirical results.}
The 10-query first run (§\ref{sec:first-results}) and the 30-query stale-version attack (§\ref{sec:stale-attack}) are reproduced end-to-end with three commands against the release tag above:

\begin{lstlisting}
git clone https://github.com/PromptOwl/contextnest-eval.git
cd contextnest-eval && make build
make run                              # E1 first run
QUERIES_FILE=/workspace/queries-stale.yaml \
  OUTPUT_PREFIX=stale-attack make run # stale-version attack
\end{lstlisting}

The harness reads \texttt{ANTHROPIC\_API\_KEY} from \texttt{.env} and writes per-query results (CSV) and the headline charts (PNG) to \texttt{outputs/}. Selector retrieval is executed by invoking the published \texttt{ctx resolve} command against the fixture vault, so any conforming implementation of the specification can be substituted for the reference implementation by overriding the \texttt{ctx} binary on \texttt{PATH}.

\paragraph{Verification commands.}
The integrity-verification claims of §\ref{sec:integrity} are exercised by:

\begin{lstlisting}
ctx verify                  # full vault chain-hash verification
ctx history nodes/<id>      # per-document version chain with hashes
ctx checkpoint list         # the append-only checkpoint log
\end{lstlisting}

\section{Discussion}
\label{sec:discussion}

\subsection{Limitations}
\label{sec:limitations}

\textbf{Access control is layered, not monolithic.} ContextNest separates the \emph{publication gate} (status, §\ref{sec:document}) from the \emph{authorization layer} (stewardship, §\ref{sec:stewardship}). The status mechanism is part of the format specification and is enforced by every conforming reader. Stewardship---role assignment, scope resolution, and separation-of-duties---is part of the platform layer and is implemented by the reference server but is not strictly required of every implementation. A minimal client that reads the vault directly from disk sees only the published-vs-draft distinction; richer clients enforce the full role lattice. We view this as an appropriate division: the format-level guarantee is universal, the platform-level guarantee is opt-in, and a vault is portable across both.

\textbf{Local-first architecture.} ContextNest vaults are directory-based and designed for local-first usage. Multi-user collaboration, real-time editing, and conflict resolution require a platform layer not defined by the specification.

\textbf{Author attribution is at the document level.} The frontmatter \texttt{author} field of §\ref{sec:document} records a single principal per version: the primary maintainer responsible for that version of the document. The chain hash of §\ref{sec:integrity} binds this attribution to the version's content, integrity-protecting it against rewrite. The model does not, however, support \emph{line-level} or \emph{block-level} attribution---the analogue of \texttt{git blame}---in v1 of the specification. Practical knowledge documents are often co-authored at the paragraph or section level, and a future revision of the specification is expected to extend the attribution model to a piecewise form (attribution records bound to anchor ranges in the document body, integrity-protected as part of the same chain hash). The interim convention is that the frontmatter \texttt{author} is the principal of record for the version as a whole; secondary contributors and their contributions are captured in the document body (acknowledgments, change-log section, or inline byline) at the author's discretion. This is a deliberate scope boundary for v1, not an oversight.

\textbf{Semantic retrieval is delegated, not built in.} The selector grammar (§\ref{sec:selectors}) and the direct/set addressing classes of §\ref{sec:resolution-classes} cover deterministic structural selection. Similarity-based and embedding-based retrieval are addressed through the delegated-retrieval surface of §\ref{sec:resolution-classes} (\texttt{contextnest://search/\{query\}}), whose resolver is plugged in via a configured back-end---a RAG pipeline, a hybrid sparse+dense index, or any conforming retriever. The specification does not include a built-in semantic retriever; the design choice (rather than the absence) is to keep the governance layer independent of any particular retrieval algorithm so that the two compose cleanly (§\ref{sec:complementary}). For users who require similarity search without configuring a separate back-end, the appropriate composition with an off-the-shelf RAG layer is straightforward and documented in §\ref{sec:complementary}.

\textbf{Hash chain verification is retrospective.} The integrity mechanism detects tampering after the fact but does not prevent it. An actor with write access to the vault can modify both content and hashes. The hash chain provides evidence of tampering (broken chains), not prevention.

\textbf{Agent identity is out of scope.} The current specification governs the \emph{knowledge} an agent consumes but does not govern \emph{which agent} is consuming it. Cryptographic agent identity, owner attestation, and revocation---requirements that are foundational in cross-organizational agentic commerce~\citep{konsynski2026trustworthy, mastercard2025agentpay, ap22025}---are deferred to a companion specification. The audit trace of §\ref{sec:injection} records the principal making each request but does not itself attest to the principal's identity outside the local trust domain.

\subsection{Comparison with Existing Approaches}

\begin{table}[h]
\centering
\small
\begin{tabular}{@{}lcccc@{}}
\toprule
\textbf{Property} & \textbf{RAG} & \textbf{KGs} & \textbf{Git} & \textbf{ContextNest} \\
\midrule
Provenance            & No  & Partial & Yes & Yes \\
Version identity      & No  & No      & Yes & Yes \\
Integrity             & No  & No      & Yes & Yes \\
Deterministic select. & No  & Yes     & N/A & Yes \\
Traceability          & No  & No      & No  & Yes \\
Temporal consistency  & No  & No      & Yes & Yes \\
Semantic retrieval    & Yes & Yes     & No  & No  \\
Knowledge preserved   & No  & No      & Yes & Yes \\
\bottomrule
\end{tabular}
\caption{Comparison of context governance properties across approaches. KGs = Knowledge Graphs. RAG includes both sparse and dense retrieval pipelines.}
\label{tab:comparison}
\end{table}

\subsection{Experimental Program}
\label{sec:experiments}

This draft argues structurally for context governance and reports first empirical results in §\ref{sec:first-results}--§\ref{sec:stale-attack}. The remaining experiments below are scheduled in priority order, each defending a specific claim made in the body. Tier-1 experiments (E1--E3) address the central claims and are partially complete; Tier-2 (E4--E6) demand more setup but are needed for comparison against established baselines; Tier-3 (E7--E9) characterize the systems-level performance of the reference implementation.

\paragraph{E1. Token cost: governed selection vs.\ retrieval-augmented baselines.} \textbf{Status: partial} (§\ref{sec:first-results}, §\ref{sec:stale-attack}). \emph{Claim defended:} governed selection yields competitive answer quality at lower input-token cost than top-$k$ retrieval, because the selector grammar surfaces exactly the relevant documents rather than over-retrieving by similarity. The full E1 grid (three retrieval conditions including dense embeddings (e.g., \texttt{bge-small-en-v1.5}), $k \in \{3, 5, 10\}$, on the 50-query suite, with inter-judge agreement reported) remains in progress. \emph{Headline output:} a Pareto curve of tokens-injected vs.\ answer quality.

\paragraph{E2. Determinism of retrieval.} \textbf{Status: partial} (§\ref{sec:determinism}). \emph{Claim defended:} Property~4 (deterministic selection). Selectors return identical results for identical queries; embedding-based retrieval does not. A 50-query run against a $1{,}060$-document synthesized corpus (with 20~reps per query per method) is reported in §\ref{sec:determinism}: the selector and BM25 baselines are perfectly deterministic on every query; the dense+HNSW baseline fails determinism on $80\%$ of queries (mean Jaccard $0.611$, worst-case $0.210$). The full E2 grid (100~queries $\times$ 100~repetitions $\times$ three retrieval conditions including hybrid sparse+dense, on multiple corpora) remains scheduled.

\paragraph{E3. Tamper detection.} \textbf{Status: scheduled.} \emph{Claim defended:} Property~3 (integrity). The hash chain detects post-publication modification. \emph{Setup:} a published vault at checkpoint~$N$, attacked under six patterns: (a) silent content edit, (b) author rewrite, (c) timestamp backdating, (d) version reordering, (e) checkpoint forgery, (f) collusive multi-document edit consistent across two related histories. \emph{Metrics:} detection rate of \texttt{ctx verify} for each pattern; granularity of the detection signal; time-to-detect on vaults of varying size. \emph{Expected:} 100\% detection for (a)--(e) by construction.

\paragraph{E4. Faithfulness under content drift.} \textbf{Status: scheduled.} \emph{Claim defended:} governed context preserves answer correctness as the underlying corpus evolves; ungoverned RAG drifts. \emph{Setup:} a governed vault and an embedding RAG index, both seeded with the same documents. Over 30~simulated days, $\sim$10\% of documents per day are edited. Each day, the same 20~questions are posed; each question's correct answer changes at some point. The RAG index re-embeds on a configurable cadence (immediate, hourly, daily). \emph{Metrics:} fraction of agent answers reflecting the currently-approved version vs.\ a stale version; auditability score.

\paragraph{E5. Adversarial poisoning.} \textbf{Status: partial} (§\ref{sec:stale-attack} reports a sharpened variant focused on stale-version leakage). \emph{Claim defended:} governed injection plus integrity verification protects against retrieval-time misinformation introduced after publication. The full E5 (six attack patterns from E3, three retrieval conditions including dense embedding, verify-before-inject pipeline) remains scheduled.

\paragraph{E6. Multi-hop QA on a public benchmark.} \textbf{Status: scheduled.} HotpotQA (distractor split) or 2WikiMultiHopQA, with selectors competing against dense retrieval on the same generator. \emph{Metrics:} exact-match, F1, faithfulness (RAGAS), tokens injected, cost per correct answer.

\paragraph{E7. Vault scaling.} \textbf{Status: scheduled.} Synthesized vaults at $10^2$ through $10^6$ nodes; p50/p99 latencies for selector evaluation, checkpoint reconstruction, and full-vault verify.

\paragraph{E8. Storage efficiency of keyframe-plus-diff.} \textbf{Status: scheduled.} Compare full-snapshot, Git loose-object, and keyframe-plus-diff at varying~$k$ on a synthesized 10-year edit history.

\paragraph{E9. Federation overhead.} \textbf{Status: scheduled.} Two- and ten-vault federations measured for resolution latency and cache hit rate against a flattened single-vault baseline.

\paragraph{E10. Small-model uplift under governed selection.} \textbf{Status: scheduled.} The task suite from E1 (50~queries), evaluated under the cross product of two retrieval conditions (selector vs.\ dense top-$k$) and three model conditions (small, mid-tier, frontier from the same family). 300 evaluations total. \emph{Headline metric:} the small-model uplift $\Delta_{\mathrm{small}} = Q_{\mathrm{small}}(\mathrm{selector}) - Q_{\mathrm{small}}(\mathrm{RAG})$ compared against $\Delta_{\mathrm{large}}$. Hypothesis: $\Delta_{\mathrm{small}} > \Delta_{\mathrm{large}}$.

\paragraph{E11. Coding benchmarks on real codebases.} \textbf{Status: scheduled.} SWE-bench Verified with three conditions: whole-repo dump, dense RAG over chunked source, and selector queries that exploit \texttt{type:code}, tag filters, and URI traversal. Variant E11b uses an agentic retrieval back-end choice.

\paragraph{Reproducibility.}
The experimental harness, fixture vaults, attack scripts, and evaluation rubrics are openly released at the project repository (§\ref{sec:artifacts}). Third parties can reproduce the headline numbers and apply the rubric to their own retrieval mechanisms.

\subsection{Other Future Work}
\label{sec:future-work}

Beyond the experimental program, several directions extend the specification itself.

\textbf{Agent identity and cryptographic attestation.} A companion specification for agent identity, owner attestation, and revocation---adjacent to the audit-trace and evidence-bundle structure of §\ref{sec:injection}---would close the gap identified in §\ref{sec:limitations}. The cross-organizational identity primitives developed for agentic commerce~\citep{ap22025, mastercard2025agentpay} are natural points of integration. We treat this as the highest-priority extension.

\textbf{Executable governed nodes.} ContextNest currently governs the \emph{knowledge} an agent reads. A natural extension is to govern the \emph{actions} an agent takes---promoting markdown nodes that describe agent skills (with tool grants, allowed hosts, and execution constraints) into first-class governed artifacts under the same approval, versioning, and integrity machinery.

\textbf{Federation protocol.} While the specification defines namespace federation semantics, the inter-vault resolution protocol (registry discovery, authentication, caching) requires further specification.

\textbf{Piecewise author attribution.} As noted in §\ref{sec:limitations}, v1 of the specification carries a single \texttt{author} per version in frontmatter. A v2 extension is expected to support piecewise attribution---records bound to anchor ranges in the document body (the same anchor space used for sub-document selection in §\ref{sec:chunking}), integrity-protected as part of the existing chain-hash construction (§\ref{sec:integrity}). The intended user experience is a \texttt{ctx blame} command analogous to \texttt{git blame}, returning the responsible principal and approval record for a specified anchor range or line span. The audit trace would then identify not just which version informed an agent's output, but which co-author's contribution within that version.

\textbf{Staged source-node lifecycle: implementation.} The staged source-node lifecycle state is now specified in §\ref{sec:staged-state}: a third lifecycle position between session-scoped hydration and durable publication that closes the audit-trail gap for ephemeral external content. The pending work is its reference implementation---the per-source-node staging chain, the session-scoped selector visibility, and the steward-driven promotion and garbage-collection paths---together with empirical validation of the retention/audit-completeness trade against live source types of differing volatility.

\textbf{Certification and continuous compliance.} The audit-evidence structure of §\ref{sec:injection} suggests a natural integration with PCI~DSS, SOC~2, and ISO/IEC~42001~\citep{iso42001} certification programs for agents and orchestration platforms. The audit log's append-only structure (§\ref{sec:integrity}) and the standard telemetry conventions referenced via OpenTelemetry~\citep{opentelemetry} provide the substrate; the certification framing is left to future work.

\textbf{Integration with training data governance.} ContextNest addresses inference-time knowledge; extending the provenance model to cover training data---connecting to datasheets \citep{gebru2021datasheets} and model cards \citep{mitchell2019model}---would provide end-to-end AI knowledge governance.

\section{Conclusion}

We have presented ContextNest, an open specification for structured, versioned, and verifiable knowledge governance for AI agents. By treating context as a first-class governed artifact---with typed documents, deterministic selection, cryptographic integrity, temporal checkpoints, and injection tracing---ContextNest addresses the context governance gap (CGG) that current RAG architectures leave open. First empirical results (§\ref{sec:first-results}--§\ref{sec:stale-attack}) demonstrate that governed selection strictly Pareto-dominates BM25 sparse retrieval on a controlled stale-version attack, achieving higher answer-quality pass rate at approximately one-third the input-token cost; two distinct BM25 failure modes (stale-version poisoning and retrieval miss) are demonstrated in the same suite.

The specification is intentionally layered: at its simplest, a ContextNest vault is a directory of Markdown files with YAML frontmatter, editable in any text editor. At its most capable, it provides hash-chained versioning, point-in-time graph reconstruction, federated cross-vault references, and complete audit trails of AI knowledge consumption.

As AI agents take increasingly autonomous actions in enterprise environments, the governance of the knowledge they consume transitions from a quality concern to a safety and compliance requirement. Retrieval is not governance. ContextNest provides the governance substrate beneath retrieval. Trustworthy autonomy will not come from better models or better tools alone; it will come from those advances composed with a verifiable knowledge supply chain. Context governance is the missing control plane for agentic systems.

The architectural frameworks that make this work possible were not invented in the 2020s; they were mapped out, tested, and debated within the Information Systems discipline more than three decades ago~\citep{elofson1991delegation, fjeldstad1986apportionment, sviokla1994reapportionment, konsynski2024cognitive}. Realizing the full potential of agentic AI requires looking past the capability frontier of the underlying models and finally building the artifact-level infrastructure that thirty years of theory has been calling for. ContextNest is one component of that infrastructure: a governed knowledge substrate that makes the principle of progressive, accountable cognitive offload---from delegation through apportionment to reapportionment of judgment---enforceable at the level of the artifacts the agents actually consume.

\section*{Author Contributions\protect\footnote{Authorship of this paper reflects scholarly contribution to the work described and does not constitute or imply any assignment, transfer, or determination of intellectual-property ownership or patent inventorship. The ContextNest specification and reference implementation are the property of PromptOwl,~LLC.}}
\label{sec:contributions}

\textbf{M.~Sulpovar} conceived ContextNest, designed the specification (including the document model, stewardship layer, selector grammar, the addressable URI scheme [co-designed with Q.~Kanchwala], hash-chain integrity construction, checkpoint mechanism, audit-trace schema, and source-node model), authored and maintains the three reference-implementation packages (engine, CLI, and MCP server), designed and ran the experimental program reported in §\ref{sec:first-results}--§\ref{sec:determinism} (including the synthesis of the $1{,}060$-document corpus used in §\ref{sec:determinism}), and produced the original draft of this paper. \textbf{B.~R.~Konsynski} contributed the revised introduction (§\ref{sec:intro}) including its thesis sentence, the ``retrieval is not governance'' formulation, the promotion of \emph{context governance gap} (CGG) to a defined and abbreviated term throughout the paper, the four-plane decomposition through which §\ref{sec:implementation} is organized (drawn from independent work on trustworthy agentic systems), the arXiv-path repositioning strategy, the Information Systems lineage that grounds §\ref{sec:intro} (Intellectual lineage paragraph), §\ref{sec:complementary}, §\ref{sec:related} (Trustworthy Agentic Systems subsection), §\ref{sec:governance-modes}, and the conclusion---the cumulative framework from delegation technologies through cognitive apportionment to cognitive reapportionment, drawn from his own four-decade research program---and substantive review of the v3 and v3.1 drafts. \textbf{Q.~Kanchwala} contributed to the reference implementation of the governed-context machinery: applying established hash-chaining and content-addressing techniques (cf.~\citep{merkle1987digital, torvalds2005git}) to realize the version-history and checkpoint constructions of §\ref{sec:integrity}--§\ref{sec:checkpoints}; co-designing the addressable \texttt{contextnest://} URI scheme (§\ref{sec:uri}) and implementing its resolution and canonicalization; and contributing to the selector syntax through paired development with the first author and to validation of the reference implementation; he reviewed the final manuscript. \textbf{G.~Goodhart} contributed the framing realignment of v6 (the ``governance frame beneath retrieval, not RAG replacement'' positioning in the abstract and §\ref{sec:intro}, the architectural-overview foreshadow in §\ref{sec:architecture-overview}, the reordering of the stewardship subsections in §\ref{sec:stewardship} so the ungoverned default is established before the governed-mode enforcement, the delegated-retrieval framing of \texttt{contextnest://search/\{query\}} in §\ref{sec:resolution-classes}, and the selector-anchored chunking pattern in §\ref{sec:chunking}); authored the specification of the staged source-node lifecycle state (§\ref{sec:staged-state}), including its session-scoped selector visibility; and contributed substantive review of the v5 draft. All authors approved the final manuscript.

\section*{Disclosure}
\label{sec:disclosure}

The ContextNest specification, the three reference-implementation packages (\texttt{@promptowl/contextnest-engine}, \texttt{@promptowl/contextnest-cli}, \texttt{@promptowl/contextnest-mcp-server}), the experimental harness, the fixture vaults, and the query suites described in this paper are the intellectual property of PromptOwl, LLC, released under the open-source licenses indicated in §\ref{sec:implementation} and §\ref{sec:artifacts}. The first author (M.~Sulpovar) has a financial interest in PromptOwl, LLC, which also develops additional software that consumes ContextNest vaults at runtime; that software is out of scope for this paper. The second author (B.~R.~Konsynski) has no financial interest in PromptOwl, LLC. The third author (Q.~Kanchwala) participates in this work in his individual capacity, on personal time, outside the scope of his employment; he has no financial interest in PromptOwl, LLC, and the work is not undertaken as work-for-hire for his employer. The fourth author (G.~Goodhart) likewise participates in this work in his individual capacity, on personal time, outside the scope of his employment; he has no financial interest in PromptOwl, LLC, and the work is not undertaken as work-for-hire for his employer.

\bibliographystyle{plainnat}

\end{document}